%% file: main.tex
\documentclass{article} %
\usepackage{iclr2027_conference,times}

\usepackage[T1]{fontenc}
\usepackage{microtype}
\usepackage{amsmath}
\usepackage{amssymb}
\usepackage{amsthm}
\usepackage{mathtools}
\usepackage{booktabs}
\usepackage{graphicx}
\usepackage{subcaption}
\usepackage{float}
\usepackage{needspace}
\usepackage{tikz}
\usepackage{array}
\usepackage{siunitx}
\usepackage{xurl}
\usepackage{parskip}
\usepackage[colorlinks=true,citecolor=blue,linkcolor=blue,urlcolor=blue]{hyperref}
\usepackage{cleveref}
\input{figs/style}
\input{experiments/results/conceptual_constants}

\newtheorem{theorem}{Theorem}
\newtheorem{proposition}{Proposition}

\theoremstyle{definition}
\newtheorem{definition}{Definition}

\newtheorem{example}{Example}
\theoremstyle{remark}

\crefname{assumption}{Assumption}{Assumptions}
\Crefname{assumption}{Assumption}{Assumptions}
\crefname{example}{Example}{Examples}
\Crefname{example}{Example}{Examples}
\AddToHook{cmd/appendix/before}{\crefalias{section}{appendix}}
\AddToHook{cmd/appendix/before}{\crefalias{subsection}{appendix}}

\DeclareMathOperator*{\argmax}{arg\,max}

\newcommand{\Atrain}{\mathcal{A}^{\mathrm{train}}}
\newcommand{\Aeval}{\mathcal{A}^{\mathrm{eval}}}
\newcommand{\E}{\mathbb{E}}
\newcommand{\X}{\mathcal{X}}
\newcommand{\Y}{\mathcal{Y}}
\newcommand{\AFAITD}{\textbf{AFA-ITD}}

\definecolor{methodaaco}{HTML}{BB5300}
\DeclareRobustCommand{\AACO}{\textcolor{methodaaco}{\textup{AACO}}}
\definecolor{methoddime}{HTML}{1777AE}
\DeclareRobustCommand{\DIME}{\textcolor{methoddime}{\textup{DIME}}}
\definecolor{methodgdfs}{HTML}{00825F}
\DeclareRobustCommand{\GDFS}{\textcolor{methodgdfs}{\textup{GDFS}}}
\definecolor{methodjafa}{HTML}{882255}
\DeclareRobustCommand{\JAFA}{\textcolor{methodjafa}{\textup{JAFA}}}
\definecolor{methodol}{HTML}{976800}
\DeclareRobustCommand{\OL}{\textcolor{methodol}{\textup{OL}}}
\definecolor{methododin}{HTML}{6F69AF}
\DeclareRobustCommand{\ODIN}{\textcolor{methododin}{\textup{ODIN}}}
\title{Active Feature Acquisition With Incomplete Training Data}
\author{Reza Rezvan, Valter Sch\"utz, Han Wu, Linus Aronsson, Morteza Haghir Chehreghani \\
Department of Computer Science and Engineering \\
Chalmers University of Technology \& University of Gothenburg \\
{\small\texttt{\{rezvan,valter.schutz,hanwu,linaro,morteza.chehreghani\}@chalmers.se}}
}

\iclrfinalcopy%
\begin{document}

\maketitle
\lhead{Preprint}
\begin{abstract}
	In many prediction tasks, acquiring all features can be a prohibitively expensive or outright impossible task.
	Further, in many cases a static subset of features may not be enough to solve the problem sufficiently across various instances.
	Active Feature Acquisition (AFA) addresses these problems by formalizing the trade-off between feature cost and predictive performance during sequential feature selection.
	However, prior AFA work largely assumes access to \emph{complete training data}, an assumption that is often violated in practice.
	Here we study \textbf{AFA} with \textbf{I}ncomplete \textbf{T}raining \textbf{D}ata (\textbf{\AFAITD}), showing that under missing completely at random (MCAR) data, one-step acquisition values remain unchanged, whereas multi-step values can decrease.
	We analyze three approaches to learning from incomplete data: \emph{aliasing, filtering, and generative restoration}.
	We show that filtering can require a number of training instances scaling exponentially with the dimension, whereas generative restoration scales exponentially with the acquisition budget.
	We empirically test our theory on a controlled experiment and across common AFA datasets and find that missingness mainly damages methods that exploit multi-step acquisitions and that generative restoration is able to recover lost performance in many experiments.
	Code is available at \url{https://github.com/Linusaronsson/AFA-Benchmark/tree/missing-data}.
\end{abstract}

\input{sections/introduction}
\input{sections/related_work}
\input{sections/background}
\input{sections/theory}
\input{sections/experiments}
\Needspace{10\baselineskip}
\input{sections/conclusion}

\subsubsection*{Acknowledgments}
The work of Valter Schütz and Morteza Haghir Chehreghani was partially supported by the Swedish Research Council VR (grant number 2023-04809).
The work of Linus Aronsson and Morteza Haghir Chehreghani was partially supported by the Wallenberg AI, Autonomous Systems and Software Program (WASP) funded by the Knut and Alice Wallenberg Foundation.
Finally, the computations and data handling were enabled by resources provided by the National Academic Infrastructure for Supercomputing in Sweden (NAISS).

\bibliographystyle{iclr2027_conference}
\bibliography{references}

\input{sections/appendix}

\end{document}

%% file: figs/style.tex
\usetikzlibrary{positioning,calc,fit,arrows.meta,decorations.pathreplacing}
\definecolor{dirblue}{HTML}{4C92D9}   %
\definecolor{agnwarm}{HTML}{B5479B}   %
\definecolor{gengreen}{HTML}{0E7A4A}  %
\definecolor{inkdark}{HTML}{0B0B0B}   %
\definecolor{inkmuted}{HTML}{52514E}  %
\definecolor{figgrid}{HTML}{D8D7D2}   %
\tikzset{
matrixbed/.style={fill=figgrid, draw=none},
obscell/.style={fill=white, draw=none},
gencell/.style={fill=gengreen!55, draw=gengreen, line width=0.25pt},
region/.style={draw=#1, line width=0.7pt, rounded corners=1pt},
evalmark/.style={circle, fill=#1, draw=none, inner sep=0pt, minimum size=3.1pt},
trainmark/.style={circle, fill=white, draw=#1, line width=0.6pt,
		inner sep=0pt, minimum size=3.3pt},
axisrule/.style={draw=inkmuted, line width=0.4pt},
tickrule/.style={draw=figgrid, line width=0.4pt},
rowlabel/.style={font=\scriptsize, text=inkdark, align=right, anchor=east},
paneltitle/.style={font=\scriptsize, text=inkdark, align=center},
rolecell/.style={fill=white, draw=figgrid, line width=0.5pt},
noisecell/.style={fill=figgrid, draw=figgrid, line width=0.5pt},
rolearrow/.style={draw=inkmuted, line width=0.45pt,
-{Straight Barb[length=2.2pt,width=2.6pt]}},
brace/.style={draw=inkmuted, line width=0.45pt, decorate,
		decoration={brace, amplitude=2.2pt, raise=1pt}},
figlabel/.style={font=\tiny, text=inkdark, align=center, inner sep=1pt},
figmuted/.style={figlabel, text=black},
}

%% file: experiments/results/conceptual_constants.tex
\newcommand{\peQEvalContextOne}{0.50}
\newcommand{\peQEvalContextTwo}{1.00}
\newcommand{\peQTrainContextTwo}{0.50}
\newcommand{\peQEvalShortcutOne}{0.75}
\newcommand{\peQEvalShortcutTwo}{0.75}
\newcommand{\avN}{6}
\newcommand{\avD}{6}

\newcommand{\avP}{0.3}
\newcommand{\avComplete}{1}

\newcommand{\avObservedCells}{1/1, 1/2, 1/4, 1/5, 1/6, 2/1, 2/4, 2/5, 3/1, 3/2, 3/3, 3/4, 3/5, 4/2, 4/3, 4/4, 4/5, 4/6, 5/2, 5/5, 5/6, 6/1, 6/2, 6/3, 6/4, 6/5, 6/6}
\newcommand{\avCompleteRows}{6}

%% file: sections/introduction.tex
\section{Introduction}
\label{sec:intro}
Obtaining information in the real world is often a laborious and costly task.
Medical tests require considerable time and resources, recommender systems need a continual collection of information at the price of compromising user privacy, and robotic systems expend energy on every sensor measurement~\citep{Gorry1968, Jeckmans2013, Lauri2023}.
Also, acquiring all features before producing a prediction may be unnecessary or impractical.
This has motivated \emph{cost-sensitive learning with feature acquisition costs}, where the goal is to balance predictive performance against the cost of acquiring features.
Further, the information required for an accurate prediction might differ between instances: the relevance of certain tests may only become apparent after observing the results of previous ones, with the appropriate sequence of tests differing between patients.

Active Feature Acquisition (AFA) models this problem as the sequential selection of features to acquire for each instance, balancing predictive performance against acquisition cost~\citep{AronssonChehreghani2026Pathwise, pmlr-v267-guney25a, norcliffe2025stochasticencodingsactivefeature}.
AFA methods are often categorized into \emph{myopic} and \emph{non-myopic} approaches, depending on whether they consider the long-term consequences of acquisitions or not~\citep{Aronsson2026}.
However, the current AFA literature has largely assumed access to \emph{complete training data}~\citep{Aronsson2026}, which is often hard to come by in practice.
For example, medical records often only contain the specific test results that were ordered for each patient; it is unrealistic to assume that every test was performed on every patient.
In contrast, during medical checkups it may be possible to request any relevant test for a patient.
This resulting mismatch between what is available during training and evaluation is a common challenge.

Motivated by this, we study \textbf{AFA} with \textbf{I}ncomplete \textbf{T}raining \textbf{D}ata (\AFAITD), formalizing it and analyzing the effect incomplete training data has on the learning process.
Our main contributions are the following.
\begin{itemize}
	\item We formulate the \AFAITD{} problem (\Cref{sec:afa-incomplete}) and characterize the difference between its training and evaluation MDPs (\Cref{sec:training-missingness}).

	\item We analyze the effects of incomplete training data on the Bellman recursions and show that, under missing completely at random (MCAR), missingness can lower non-myopic action values even when one-step values are unchanged (\Cref{prop:restriction}).

	\item We define and distinguish three approaches to learning from incomplete training data: the \emph{aliasing}, \emph{filtering}, and \emph{generative restoration} approaches, and characterize their data requirements (\Cref{thm:data-requirements}).

	\item We experimentally test our theory on a synthetic controlled study, and on eight common AFA datasets using six AFA baselines.
    In line with our theory, missingness mainly damages methods that exploit multi-step acquisitions, and generative restoration is often able to recover lost performance (\Cref{sec:experiments}).
\end{itemize}

%% file: sections/related_work.tex
\section{Related Work}
\label{sec:related-work}
In this section, we review related work, highlighting the works closest to our problem setting.

\paragraph{Active Feature Acquisition.}
AFA methods are generally characterized along two orthogonal axes.
The first axis is their \emph{planning horizon}: myopic methods only plan for the immediate next feature according to some one-step acquisition criterion.
Notable myopic AFA methods include EDDI~\citep{ma2019eddi}, \GDFS{}~\citep{covert2023learning}, and \DIME{}~\citep{gadgil2024estimating}.
In contrast, non-myopic methods plan for the potential multi-step effect an acquisition can have, often using reinforcement learning methods or other longer-horizon supervision formulations~\citep{DulacArnold2011, Shim2018, Janisch2019, AronssonChehreghani2026Pathwise}.
The second axis is whether they model the \emph{acquisition dynamics}.
Model-based methods estimate these conditional distributions and use them during training, whereas model-free methods train directly from experience and interaction. Hybrid methods combine both, learning from experience with an explicit acquisition model~\citep{Aronsson2026}.

\paragraph{AFA from Incomplete Training Data.}
To the best of our knowledge, the AFA literature on incomplete training data remains very limited.
\citet{Janisch2020} were the first to investigate the \AFAITD{} problem setting.
During training, they restrict feature-acquisition actions to available features and compare methods with mean and MICE imputation baselines.
Although their approach is empirically effective, the resulting restricted Bellman problem is not formally analyzed.
More recently, \citet{Kobayashi2026L2M} studied AFA from incomplete data and characterized specific conditions under which a greedy conditional-mutual-information acquisition objective remains identifiable.
\citet{vonKleist2023} instead investigated the evaluation of AFA policies from incomplete data using semi-offline estimators.
\citet{Wendland2026MissingnessMDPs} have also proposed miss-MDPs, a general class of POMDPs in which missingness remains part of the observation process at deployment.
None of these works have formalized the \AFAITD{} problem setting and characterized how incomplete training data affects the Bellman recursions.

\paragraph{Generative Models in AFA.}
Generative models are commonly used in AFA to approximate acquisition dynamics.
\citet{ma2019eddi} introduced the Partial Variational Autoencoder (PVAE) architecture to estimate conditionals and to evaluate a myopic information-gain acquisition criterion.
\citet{ODIN2019} later used the PVAE architecture to generate synthetic acquisition trajectories, primarily to improve data efficiency, while \citet{LiOliva2021} learned a generative surrogate model to provide intermediate rewards and auxiliary information during learning.
However, the benefit of generative modeling remains unclear in the literature.
\citet{Schutz2025} find that the gap between model-free and model-based approaches is often smaller than previously reported on common AFA datasets using a fair benchmarking setting.
In this work, we will show that using a generative model is a natural solution to incomplete training data and not only an ad-hoc extension for data efficiency.

%% file: sections/background.tex
\section{Problem Formulation}
\label{sec:background}
In this section, we introduce and formalize the AFA optimization problem, its formulation as a (Partially Observable) Markov Decision Process ((PO)MDP), the standard taxonomy of missingness mechanisms, and formulate the AFA problem with incomplete training data.

\subsection{Active Feature Acquisition (AFA)}
Active Feature Acquisition (AFA) formulates an explicit trade-off mechanism between acquiring features at a cost and their (expected) predictive performance.
Usually, this trade-off is formulated as an optimization problem.
We use the formulation of \citet{Aronsson2026} and only consider the \emph{hard-budget} classification objective.

Let $x \coloneqq (x_1, \ldots, x_d) \in \X$ be an instance with the corresponding label $y \in \Y$, distributed according to $(x, y) \sim p(x, y)$.
For $S \subseteq [d] \coloneqq \{1, \dots, d\}$ we write $x_S$ for the acquired values and $U \coloneqq [d] \setminus S$ with $x_U$ for the unacquired values.
Taking action $a \in U$ reveals feature $x_a$ and incurs cost $c_a$, and we define the accumulated cost as $c(S) \coloneqq \sum_{i \in S} c_i$.
Furthermore, a policy $\pi$ maps the current state $(x_S, S)$ to an unacquired feature $a \in U$ and a predictor $f$ maps the final observation to a label prediction.
Together, they define the following optimization problem:
\begin{equation}
	\begin{aligned}
		\min_{f, \pi} \      & \E_{(x, y)} \E_{\pi} [\ell(f(x_{\pi[x]}, \pi[x]), y)] \\
		\text{subject to} \  & c(\pi[x]) \leq b \quad \text{for all } x,
	\end{aligned}
	\label{eq:afaopt}
\end{equation}
where $\ell$ is a loss function, $b$ is the \emph{hard budget}, and $\pi[x] \subseteq [d]$ is the final set of acquired features for instance $x$ under policy $\pi$.
Throughout, we assume access to a well-trained predictor $f$, since learning the predictor is not the objective of interest in this work.

\subsection{AFA as a (Partially Observable) Markov Decision Process}
\label{sec:pomdp-mdp}
AFA can be modeled as a Partially Observable Markov Decision Process (POMDP) whose latent state contains the fully observed instance $x$ and the label $y$, and whose observations are the features acquired so far $x_S$~\citep{Aronsson2026}.
Because the latent state does not change over time, the POMDP can be reduced to a fully observable MDP on the state $s=(x_S, S)$~\citep{DulacArnold2011}.
In this fully observable MDP, the acquisition dynamics can be written in terms of the individual feature probabilities
\begin{align}
	p(s^{\prime} \mid s,a)
	 & =p(x_{S\cup\{a\}},S\cup\{a\}\mid x_S,S,a) \\
	 & =p(x_a\mid x_S,S)
	\label{eq:eval-transition}
\end{align}
where $s^{\prime} \coloneqq (x_{S \cup \{a\}}, S \cup \{a\})$. The quality of a state $s$ and an acquisition $a$ is captured by truncated (action-)value functions, for which the following Bellman recursions~\citep{Sutton} hold:
\begin{equation}
	\begin{aligned}
		V_0(s)    & = \E_{y\mid s}[-\ell(f(s),y)],                                         \\
		Q_k(s, a) & = \E_{x_a \mid s} [V_{k - 1}(s')],                                     \\
		V_k(s)    & = \max\left\{V_0(s),\ \max_{a \in \mathcal{A}_b(S)} Q_k(s, a)\right\}.
	\end{aligned}
	\label{eq:bellman}
\end{equation}

Here $V_k$ is the optimal value when at most $k$ further acquisitions are allowed and $\mathcal{A}_b(S) \coloneqq \{a \in U : c(S) + c_a \leq b\}$. Finally, to measure the gap from an optimal policy one usually uses the \emph{regret}.
Let $ \mathcal L(\pi) \coloneqq \E_{(x,y)}\E_\pi[\ell(f(x_{\pi[x]}, \pi[x]), y)]$ be the expected loss, and let $\pi^*$ be the optimal policy that minimizes $\mathcal L(\pi)$.
The regret of $\pi$ is then $ \operatorname{Reg}(\pi) \coloneqq \mathcal L(\pi)-\mathcal L(\pi^*)$.

\subsection{Missingness Mechanisms}
Let $m \in \{0,1\}^d$ be a binary missingness mask of dimension $d$, where $m_i = 1$ indicates that feature $i$ is missing from the training instance.
The missingness mask is drawn once per instance from $p(m \mid x, y)$ and remains fixed for that instance.
We write $M \coloneqq \{i : m_i = 1\}$ and $\bar M \coloneqq \{i : m_i = 0\}$ for the missing and available indices, with the corresponding feature values $x_M$ and $x_{\bar M}$, respectively.
Classically, Rubin's taxonomy~\citep{Rubin1976, Little2002} classifies mechanisms as

\begin{itemize}
	\item \textbf{M}issing \textbf{C}ompletely \textbf{A}t \textbf{R}andom (MCAR): the missingness is independent of the instance, $p(m \mid x, y) = p(m)$.
	\item \textbf{M}issing \textbf{A}t \textbf{R}andom (MAR): the missingness may depend on the available features, but not the missing features, $p(m \mid x, y) = p(m \mid x_{\bar M})$.
	\item \textbf{M}issing \textbf{N}ot \textbf{A}t \textbf{R}andom (MNAR): the missingness may depend on the missing feature values themselves, either through a feature's own value (self-masking MNAR) or through other unavailable features (e.g., logistic MNAR).
\end{itemize}

\subsection{AFA with Incomplete Training Data (\AFAITD)}
\label{sec:afa-incomplete}
In the \textbf{\AFAITD} problem setting, missing features cannot be part of the \emph{training action space}: for a training instance with missingness mask $m$, we define the training action space as $\Atrain_b(S, m) \coloneqq \{a \in U : m_a = 0,\ c(S) + c_a \leq b\}$.
In contrast, the evaluation action space does not have any missing features and is defined as $\Aeval_b(S) \coloneqq \{a \in U : c(S) + c_a \leq b\}$.
Consequently, $\Atrain_b(S, m) \subseteq \Aeval_b(S)$.
Thus, the training Bellman recursion can be written as
\begin{equation}
	\begin{aligned}
		V_0^{\mathrm{train}}(s, m)    & = \E_{y \mid s, m}[-\ell(f(s), y)],                                                                          \\
		Q_k^{\mathrm{train}}(s, m, a) & = \E_{x_a \mid s, m} \left[V_{k - 1}^{\mathrm{train}}(s^{\prime}, m)\right],                                 \\
		V_k^{\mathrm{train}}(s, m)    & = \max\left\{V_0^{\mathrm{train}}(s, m), \max_{a \in \Atrain_b(S, m)}Q_k^{\mathrm{train}}(s, m, a) \right\}.
	\end{aligned}
	\label{eq:train-bellman}
\end{equation}
These recursions hold under MCAR because the missingness mask $m$ is independent of $(x, y)$, leaving the conditionals unchanged.
By contrast, the evaluation Bellman equations remain unchanged and are defined as in \Cref{eq:bellman}.

%% file: sections/theory.tex
\section{Theoretical Analysis}
\label{sec:theory}
In this section, we analyze the effect incomplete training data can have on the Bellman recursions between training and evaluation.
We then analyze the \emph{aliasing} approach and how omitting the missingness mask from the state affects the learning process.
Afterwards, we consider the \emph{filtering} approach, through state augmentation with the missingness mask $m$, and the \emph{generative restoration} approach.
Finally, we theoretically compare their data requirements.
All proofs can be found in \Cref{app:proofs}.

\subsection{How the (Action-)Value Functions Are Affected by Incomplete Training Data}

\label{sec:training-missingness}
Inspecting \Cref{eq:train-bellman} and \Cref{eq:bellman} side by side, we see that the only difference between them is their respective action spaces.
To better understand what effect this can have, we introduce a simple non-myopic AFA problem with binary features.
\begin{example}[The Shortcut Problem]
	\label{ex:running}
	Consider $d \geq 2b$ binary features $x_i \in \{0, 1\}$ with a hard budget of $b \geq 2$.
	The features consist of the \emph{context} feature $x_1$, the two \emph{branch} blocks of features $B_1 = \{2, \ldots, b\}, B_2 = \{b + 1, \ldots, 2b - 1\}$, with \emph{block length} $b - 1$, and the \emph{shortcut} feature $x_{2b}$.
	The corresponding label is defined as $y = \bigoplus_{j \in B_{x_1 + 1}} x_j$, where $\oplus$ denotes the XOR operation.
	The shortcut feature $x_{2b}$ is equal to the label $y$ with probability $3/4$ but has a cost of $b$, while all other features have a cost of one.
	The remaining features $x_{2b + 1}, \ldots, x_d$ are independent uniform \emph{noise}.
\end{example}

\Cref{ex:running} serves as a simple non-myopic AFA problem. The optimal policy is to always first acquire the context feature $x_1$, then acquire all the features in the corresponding branch block. This yields a total cost of $b$ and an accuracy of $1$. In contrast, the shortcut also costs $b$ but only attains an accuracy of $3/4$.

\input{figs/planning_effect}

\Cref{fig:planning-effect-a} illustrates the problem for $b = 2$ and $d = 6$ for one instance where the two branch features $x_2, x_3$ are missing.
We now analyze the effect of the mismatch between training and evaluation in \Cref{fig:planning-effect-a} on the (action-)value functions.
For our analysis, we use the zero-one loss function $\ell(f(s), y) = \mathbf{1}\{f(s) \neq y\}$ and report the action values as $1 + Q(\ldots)$, corresponding to expected accuracy under zero-one loss, for easier interpretation.

For a myopic policy, the context feature $x_1$ has no immediate predictive value and the best prediction after acquiring it is a guess.
In contrast, the shortcut feature is preferable for myopic methods because it has immediate predictive value, predicting the label with probability $\frac{3}{4}$.
Thus, we have
\begin{align*}
	Q_1^{\mathrm{train}}(s_0, m, a=1) = Q_1^{\mathrm{eval}}(s_0, a=1) & = \frac{1}{2}, \\
	Q_1^{\mathrm{train}}(s_0, m, a=4) = Q_1^{\mathrm{eval}}(s_0, a=4) & = \frac{3}{4},
\end{align*}
where $s_0 \coloneqq (x_{\emptyset}, \emptyset)$ is the initial state.
In contrast, a non-myopic policy can plan for the full acquisition of both the context feature and the corresponding branch block features.
However, in \Cref{fig:planning-effect-a} neither $x_2$ nor $x_3$ is available, which means that acquiring the context feature $x_1$ will yield poorer performance than the shortcut feature $x_4$, leaving $x_1$ as good as a guess.
Consequently,
\begin{align*}
	Q_2^{\mathrm{train}}(s_0, m, a=1) = \frac{1}{2} \neq
	Q_2^{\mathrm{eval}}(s_0, a=1)                                     & = 1,           \\
	Q_2^{\mathrm{train}}(s_0, m, a=4) = Q_2^{\mathrm{eval}}(s_0, a=4) & = \frac{3}{4}.
\end{align*}

\Cref{fig:planning-effect} illustrates
that the restricted training action space can impact the learned action values, and the following proposition states how this holds more generally for an arbitrary planning horizon.

\begin{proposition}[(Action-)values under the restricted action space]
	\label{prop:restriction}
	Assume that the missingness mask $m$ is independent of $(x, y)$ (MCAR).
	Then, for every state $s = (x_S, S)$ reachable under mask $m$ and every $k \geq 0$,
	\[
		V_k^{\mathrm{train}}(s, m) \leq V_k^{\mathrm{eval}}(s),
	\]
	with equality when $k = 0$.
	For every $a\in\Atrain_b(S, m)$ and $k \geq 1$,
	\[
		Q_k^{\mathrm{train}}(s, m, a) \leq Q_k^{\mathrm{eval}}(s, a),
	\]
	with equality when $k = 1$.
\end{proposition}

Thus, we are left with a potential gap in the (action-)value functions that can impact performance. Hereafter, we seek to answer the following two questions: (i) what happens if we still try to learn the (action-)value functions from finite training data with varying missingness masks? and (ii) what are some possible mitigation methods that ensure we learn the correct evaluation problem?

\subsection{Aliasing, Filtering, and Restoration}
\Cref{prop:restriction} does not guarantee that the restricted action space results in the wrong (action-)value functions, only that it is possible. However, it is clear that if we omit the missingness mask $m$ from our state representation, we will suffer from an \emph{aliasing} issue.

\begin{definition}[Tabular Aliasing]
	The tabular \emph{aliasing} approach learns one (action-)value $Q_k^{\mathrm{alias}}(s,a)$, $V_k^{\mathrm{alias}}(s)$, or a policy $\pi^{\mathrm{alias}}(a \mid s)$, omitting the missingness mask $m$ entirely.
	\label{def:aliasing}
\end{definition}

Across instances with different missingness masks $m$ and corresponding action spaces $\Atrain_b(S, m)$, the aliasing approach will assign different values to the \emph{same} state-action pair $(s, a)$, pooling them. We can observe this issue in \Cref{fig:planning-effect}.
Conditional on the context feature $x_1$ being observed, the relevant branch feature $x_2$ or $x_3$ is available with probability $1 - p$. The corresponding action-value for acquiring the context feature $x_1$ is then
\begin{equation}
	Q_2^{\mathrm{alias}}(s_0, a = 1) = (1 - p) \cdot 1 + p \cdot \frac{1}{2} = 1 - \frac{p}{2}.
	\label{eq:alias-shortcut}
\end{equation}
The aliasing approach will therefore prefer the shortcut feature $x_4$ whenever $p > 1/2$, even though the optimal evaluation policy should always acquire the context feature $x_1$.
This error remains even with unlimited training data.

In contrast, if we augment the state representation with the missingness mask $m$, we can learn separate values for different masks.

\begin{definition}[Tabular Filtering]
	\label{def:filtering}
	The tabular \emph{filtering} approach augments the state with the missingness mask $(s, m)$ when learning $Q_k^{\mathrm{filter}}(s, m, a)$, $V_k^{\mathrm{filter}}(s, m)$, or a policy $\pi^{\mathrm{filter}}(a \mid s, m)$.
\end{definition}

During evaluation, the filtering approach will only use the learned values for the missingness mask $m = 0$, meaning that only \emph{complete instances} are used for the values.
Thus, provided that complete instances occur with positive probability, the filtering approach will be able to fully recover the evaluation Bellman problem with unlimited data.
However, this approach can be statistically inefficient: under independent MCAR, a training instance is complete with probability $(1 - p)^d$, scaling exponentially with the dimension $d$.
Consequently, only (an expected) $n(1 - p)^d$ of the $n$ training instances contribute directly.

An alternative approach is to restore the unavailable acquisition dynamics $p(x_a \mid s, m)$ using a \emph{generative model}, thereby recovering \Cref{eq:bellman} from \Cref{eq:train-bellman}.

\begin{definition}[Tabular Generative Restoration]
	\label{def:generative}
	The tabular \emph{generative restoration} approach learns acquisition dynamics $\hat p(x_a \mid s)$ and stopping values $\widehat V_0(s)$ from all instances with the features of $s$ available, and uses them to restore unavailable transitions when learning $Q_k^{\mathrm{restore}}(s,a)$, $V_k^{\mathrm{restore}}(s)$, or a policy $\pi^{\mathrm{restore}}(a \mid s)$.
\end{definition}

Note that under independent MCAR, the empirical acquisition dynamics $\hat p(x_a \mid s)$ will be consistent because the missingness is independent of the corresponding feature and label values, i.e., $p(m \mid x, y) = p(m)$.
Thus, if all features are observable with positive probability, the restoration approach can fully recover the evaluation conditionals with unlimited data, solving the exact evaluation problem (\Cref{eq:bellman}).

\subsection{Usable Instances and Data Requirements for Filtering and Restoration}
\label{sec:empirical-planning}
Both the filtering and the generative restoration approaches solve the correct evaluation Bellman recursions under MCAR and unlimited data.
However, with a finite amount of training data, the approaches will have different numbers of \emph{usable} training instances.

\input{figs/availability}

\Cref{fig:availability} illustrates the difference in usable training instances for all three approaches.
Both the aliasing and generative restoration approaches, the latter after restoring all missing features, can use all instances for learning, while filtering is only able to use complete instances.
In the tabular setting, each approach can be viewed as solving a distinct Bellman problem: aliasing solves the \emph{restricted} Bellman recursions (\Cref{eq:train-bellman}), generative restoration solves the approximate evaluation Bellman recursions, and filtering solves the exact evaluation Bellman recursions (\Cref{eq:bellman}).
Since generative restoration and filtering solve the correct Bellman recursions, we are interested in their data requirements to achieve a certain regret with high probability.
The following theorem quantifies this.

\begin{theorem}[Data requirements under MCAR]
	\label{thm:data-requirements}
	Consider $d$ discrete features $x_a\in\mathcal X_a$ with $|\mathcal X_a|\leq K$ for an integer $K\geq2$, and costs $c_a \geq 1$ under the integer hard budget $1 \leq b \leq d$.
	Let $f$ be a fixed predictor, independent of the $n$ i.i.d. training instances, with bounded loss $\ell \in [0, \ell_{\max}]$.
	Assume MCAR, with each feature independently missing with probability $p<1$.
	Let $\pi^{\mathrm{filter}}$ and $\pi^{\mathrm{restore}}$ be the exact corresponding policies obtained by their respective Bellman recursions.
	For every $0 < \varepsilon \leq \ell_{\max}$ and $\delta\in(0,1)$, filtering satisfies $\operatorname{Reg}(\pi^{\mathrm{filter}})\leq\varepsilon$ with probability at least $1-\delta$ whenever
	\[
		n \geq C \frac{\ell_{\max}^2 K^b}{\varepsilon^2 (1-p)^d}
		\left(b \log(Kd) + K + \log\frac{2}{\delta}\right).
	\]
	In contrast, generative restoration satisfies $\operatorname{Reg}(\pi^{\mathrm{restore}})\leq\varepsilon$ with probability at least $1 - \delta$ whenever
	\[
		n \geq C \frac{\ell_{\max}^2 K^b}{\varepsilon^2 (1-p)^b}
		\left(b \log(Kd) + K + \log \frac{2}{\delta}\right),
	\]
	where $C > 0$ is a numerical constant independent of all of the above parameters.
\end{theorem}

Thus, the filtering approach has an exponential dependency on the dimension $d$, whereas the generative restoration approach has an exponential dependency on the acquisition budget $b$.
The following proposition shows that the dependency on $d$ is necessary for filtering.

\begin{proposition}[Filtering requires complete instances]
	\label{prop:filter-lower}
	Consider \Cref{ex:running} with $d \geq 2b + 1$ under the zero-one loss, where each feature is independently missing with probability $p$.
	Let filtering use fixed fallback estimates when a state has no usable instances (\Cref{sec:thm_proof_step1}) and a fixed tie-breaking rule.
	Then, for some relabeling of the unit-cost features, with probability at least $1 - n(1-p)^d$,
	\[
		\operatorname{Reg}(\pi^{\mathrm{filter}}) \geq \frac{1}{4}.
	\]
	Consequently, for $\varepsilon < 1/4$, $\operatorname{Reg}(\pi^{\mathrm{filter}}) \leq \varepsilon$ with probability at least $1 - \delta$ requires $n \geq (1-\delta)(1-p)^{-d}$.
\end{proposition}

%% file: figs/planning_effect.tex
\begin{figure}[H]
	\centering
	\begin{subfigure}[b]{0.47\textwidth}
		\centering
		\begin{tikzpicture}[x=1cm,y=1cm]
			\def\cw{0.38}
			\def\ch{0.29}
			\def\gp{0.44}

			\foreach \row/\name/\height in {
			0/{Evaluation}/0.66,
			1/{Training}/0.14}{
			\node[rowlabel,anchor=east] at (-0.14,\height-0.145) {\name};
			\foreach \i/\lab in {0/{$x_1$},1/{$x_2$},2/{$x_3$},3/{$x_4$},4/{$x_5$},5/{$x_6$}}{
			\draw[rolecell] ({\i*\gp},\height) rectangle ++(\cw,-\ch);
			\node[figlabel] at ({\i*\gp+\cw/2},{\height-\ch/2}) {\lab};
			}
			\draw[rolecell] (2.72,\height) rectangle ++(\cw,-\ch);
			\node[figlabel] at ({2.72+\cw/2},{\height-\ch/2}) {$y$};
			}

			\foreach \i in {1,2}{
					\fill[figgrid] ({\i*\gp},0.14) rectangle ++(\cw,-\ch);
					\node[figlabel,text=inkmuted] at ({\i*\gp+\cw/2},{0.14-\ch/2}) {$\times$};
				}

			\draw[rolearrow] (0.19,0.73) -- (0.19,0.93) -- (1.09,0.93) -- (1.09,0.73);
			\node[figmuted,anchor=south] at (0.54,0.96) {$y=x_2$ or $y=x_3$};
			\draw[rolearrow] (1.51,0.73) -- (1.51,1.19) -- (2.91,1.19) -- (2.91,0.73);
			\node[figmuted,anchor=south] at (2.21,1.22) {$P(x_4=y)=\frac34$};
		\end{tikzpicture}
		\caption{A realization of an instance of the shortcut problem where the branch features $x_2$ and $x_3$ are missing.}
		\label{fig:planning-effect-a}
	\end{subfigure}
	\hfill
	\begin{subfigure}[b]{0.50\textwidth}
		\centering
		\begin{tikzpicture}[x=1cm,y=1cm]
			\def\xz{0.00}
			\def\xs{3.10}
			\newcommand{\vx}[1]{{\xz+(#1-0.5)*(\xs-\xz)/0.5}}
			\newcommand{\pairedmark}[2]{%
				\node[trainmark=inkdark] at (\vx{#1},#2) {};
				\node[evalmark=inkdark,minimum size=2.1pt] at (\vx{#1},#2) {};
			}

			\foreach \v in {0.50,0.75,1.00}{
					\draw[tickrule] (\vx{\v},1.62) -- (\vx{\v},-0.26);
					\node[figmuted,anchor=north] at (\vx{\v},-0.32) {$\v$};
				}
			\draw[axisrule] (\xz,-0.26) -- (\xs,-0.26);
			\node[figmuted,anchor=north] at ({(\xz+\xs)/2},-0.64) {$1 + Q_k$};

			\node[figlabel,anchor=east] at (-1.48,1.20) {$k=1$};
			\node[rowlabel] at (-0.14,1.42) {Acquire $x_1$};
			\pairedmark{\peQEvalContextOne}{1.42}
			\node[rowlabel] at (-0.14,0.98) {Acquire $x_4$};
			\pairedmark{\peQEvalShortcutOne}{0.98}

			\node[figlabel,anchor=east] at (-1.48,0.16) {$k=2$};
			\node[rowlabel] at (-0.14,0.38) {Acquire $x_1$};
			\draw[inkdark,line width=0.8pt]
			(\vx{\peQTrainContextTwo},0.38) -- (\vx{\peQEvalContextTwo},0.38);
			\node[trainmark=inkdark] at (\vx{\peQTrainContextTwo},0.38) {};
			\node[evalmark=inkdark] at (\vx{\peQEvalContextTwo},0.38) {};
			\node[rowlabel] at (-0.14,-0.06) {Acquire $x_4$};
			\pairedmark{\peQEvalShortcutTwo}{-0.06}

			\node[evalmark=inkdark] at (-1.42,1.88) {};
			\node[figmuted,anchor=west] at (-1.34,1.88) {Evaluation};
			\node[trainmark=inkdark] at (0.13,1.88) {};
			\node[figmuted,anchor=west] at (0.21,1.88) {Training};
		\end{tikzpicture}
		\caption{The expected accuracy ($1 + Q_k$) for the corresponding instance in \Cref{fig:planning-effect-a}.}
		\label{fig:planning-effect-b}
	\end{subfigure}
	\caption{The shortcut problem for a fixed missingness mask $m$ where the branch features $x_2$ and $x_3$ are missing \textbf{(a)}, and the unchanged one-step and differing two-step $1 + Q_k$ optimal values for both training and evaluation \textbf{(b)}.}
	\label{fig:planning-effect}
\end{figure}
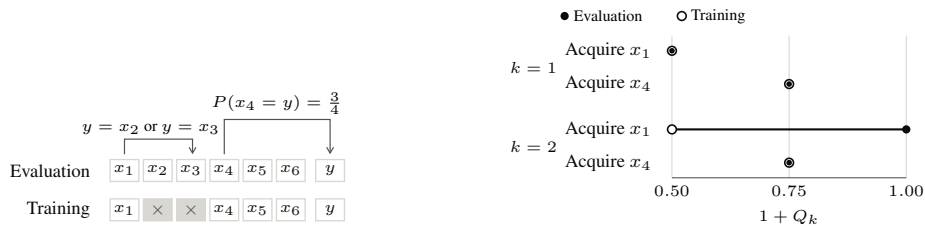

%% file: figs/availability.tex
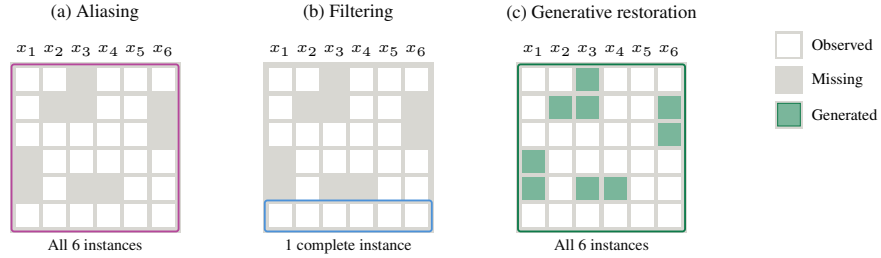
\begin{figure}[H]
	\centering
	\begin{tikzpicture}[x=1cm, y=1cm]
		\def\cell{0.30}
		\def\pitch{0.36}
		\def\mw{2.10}                       %
		\pgfmathsetmacro{\mh}{(\avN-1)*0.36+0.30}

		\newcommand{\availpanel}[4]{%
		\node[paneltitle, anchor=south] at ({#1+\mw/2},{0.50}) {#2};
		\foreach \c/\lab in {1/{$x_1$},2/{$x_2$},3/{$x_3$},4/{$x_4$},5/{$x_5$},6/{$x_6$}}
		\node[figmuted, anchor=south] at ({#1+(\c-1)*\pitch+\cell/2},{0.13}) {\lab};
		\fill[matrixbed] ({#1-0.075},{0.075}) rectangle ({#1+\mw+0.075},{-\mh-0.075});
		\foreach \r/\c in \avObservedCells
		\fill[obscell] ({#1+(\c-1)*\pitch},{-(\r-1)*\pitch}) rectangle ++(\cell,-\cell);
		\node[figlabel, anchor=north] at ({#1+\mw/2},{-\mh-0.12}) {#3};
		\node[figlabel, anchor=north] at ({#1+\mw/2},{-\mh-0.40}) {#4};
		}
		\newcommand{\availregion}[6]{%
			\draw[region=#2]
			({#1+(#3-1)*\pitch-0.055},{-(#5-1)*\pitch+0.055}) rectangle
			({#1+(#4-1)*\pitch+\cell+0.055},{-(#6-1)*\pitch-\cell-0.055});
		}

		\availpanel{0.00}{(a) Aliasing}{All \avN{} instances}{}
		\availregion{0.00}{agnwarm}{1}{\avD}{1}{\avN}

		\availpanel{3.35}{(b) Filtering}{\avComplete{} complete instance}{}
		\foreach \r in \avCompleteRows { \availregion{3.35}{dirblue}{1}{\avD}{\r}{\r} }

		\availpanel{6.70}{(c) Generative restoration}{All \avN{} instances}{}
		\foreach \r in {1,...,\avN}
		\foreach \c in {1,...,\avD}
		\fill[gencell, draw=none] ({6.70+(\c-1)*\pitch},{-(\r-1)*\pitch})
		rectangle ++(\cell,-\cell);
		\foreach \r/\c in \avObservedCells
		\fill[obscell] ({6.70+(\c-1)*\pitch},{-(\r-1)*\pitch})
		rectangle ++(\cell,-\cell);
		\availregion{6.70}{gengreen}{1}{\avD}{1}{\avN}

		\fill[matrixbed] (10.05,0.49) rectangle ++(0.36,-0.36);
		\fill[obscell] (10.08,0.46) rectangle ++(\cell,-\cell);
		\node[figmuted, anchor=west] at (10.49,0.31) {Observed};
		\fill[matrixbed] (10.05,0.03) rectangle ++(0.36,-0.36);
		\node[figmuted, anchor=west] at (10.49,-0.15) {Missing};
		\fill[matrixbed] (10.05,-0.43) rectangle ++(0.36,-0.36);
		\fill[gencell] (10.08,-0.46) rectangle ++(\cell,-\cell);
		\node[figmuted, anchor=west] at (10.49,-0.61) {Generated};
	\end{tikzpicture}
	\caption{The \emph{usable} training instances for each of the approaches, on \Cref{ex:running} under MCAR with $p = \avP$.}
	\label{fig:availability}
\end{figure}

%% file: sections/experiments.tex
\section{Experiments}
\label{sec:experiments}
In this section, we solve and evaluate the synthetic controlled study of \Cref{ex:running} to isolate the dimension-budget dependency in \Cref{thm:data-requirements}.
Then, we benchmark \emph{restricted training} (training in which acquisitions are restricted to the features available in each instance) against generative restoration on eight common AFA datasets with synthetic missingness, and connect the results to the theory.

\subsection{Synthetic Controlled Study (\Cref{ex:running})}
\label{sec:exp-exact}
To test the three approaches and our theory, we solve \Cref{ex:running} exactly.
The problem is parameterized by both the dimension $d$ and the acquisition budget $b$, so we sweep the dimension over $d \in \{6, 8, 10\}$ with a fixed acquisition budget of $b = 2$ and sweep the budget over $b \in \{2, 3, 4\}$ with a fixed dimension of $d = 10$, both under MCAR with missingness rate fixed at $p = 0.6$.
The approaches first estimate the acquisition dynamics $\widehat p(x_a \mid s)$ and the expected terminal losses $\widehat V_0(s)$ from the usable training instances with a fixed Bayes predictor, using a pseudocount of $1/2$ (\Cref{app:pseudocount}).
Then, they solve their respective Bellman problem exactly with dynamic programming using these estimates.
Finally, they are evaluated against the optimal oracle policy in terms of evaluation regret.

\begin{figure}[!tb]
	\centering
	\includegraphics[width=\textwidth]{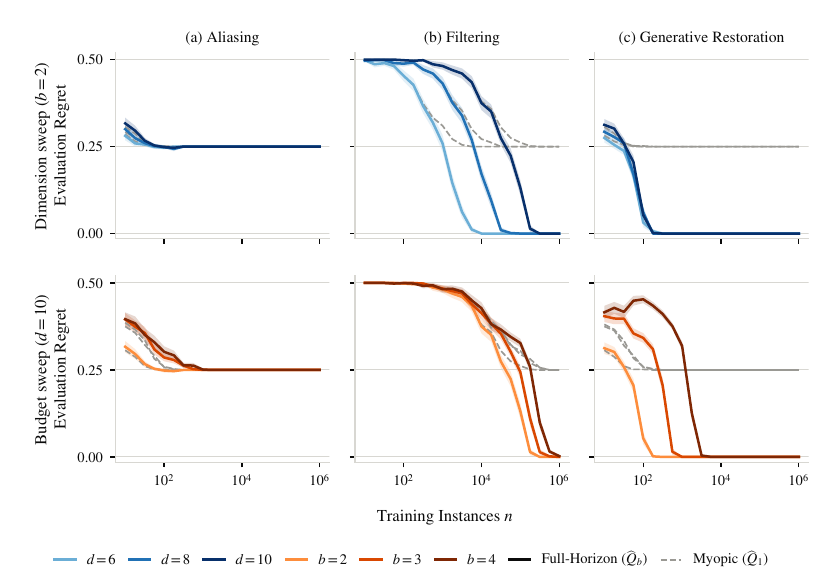}
	\caption{Evaluation regret of the myopic $\widehat Q_1$ (gray, dashed) and full planning-horizon $\widehat Q_b$ (solid) methods, across the three learning approaches swept over the dimension $d$ (top row) and the acquisition budget $b$ (bottom row), under MCAR at $p = 0.6$, with 95\% confidence bands.}
	\label{fig:exact-study}
\end{figure}

\Cref{fig:exact-study} illustrates the myopic policy $\widehat Q_1$ along with full-horizon planning $\widehat Q_b$ on the same instances over 150 independent replicates, each with a freshly sampled training set and missingness masks, of the three approaches for \Cref{ex:running}.
As expected from the one-step values in \Cref{sec:training-missingness}, the myopic policy prefers the shortcut and only achieves a regret of $1/4$ across all configurations.
By contrast, the full-horizon filtering and restoration policies achieve zero regret given enough training data.
In the dimension sweep, the filtering approach becomes more data-intensive as $d$ increases, whereas restoration converges at the same training size across the tested dimensions.
In contrast, in the budget sweep, increasing $b$ lengthens the optimal acquisition route and consequently increases restoration's data requirements at fixed $d$, reflecting what is described in \Cref{thm:data-requirements} in both sweeps.
Finally, the aliasing approach converges to the shortcut and a regret of $1/4$, since the relevant branch block is available with probability $(1-p)^{b-1} < 1/2$ (\Cref{eq:alias-shortcut}).

\Cref{fig:exact-collapse} plots the full-horizon regret of the dimension sweep against the effective number of training instances in \Cref{thm:data-requirements}, $n(1-p)^d$ for filtering and $n(1-p)^b$ for restoration, at $p \in \{0.3, 0.5, 0.7\}$.
The curves of all dimensions and missingness rates nearly coincide for both approaches.

\begin{figure}[!tb]
	\centering
	\includegraphics[width=\textwidth]{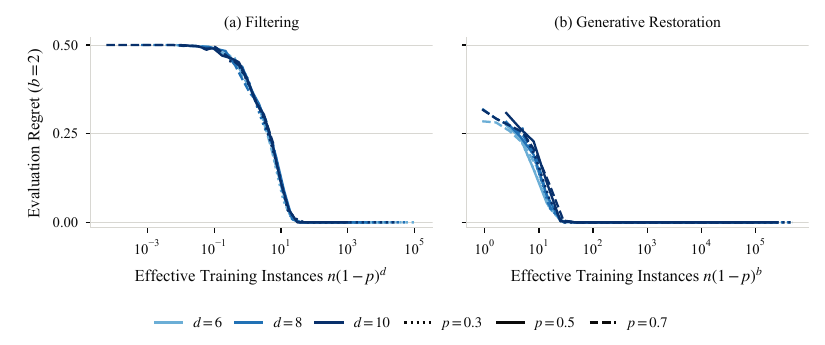}
	\caption{Evaluation regret of the full planning-horizon $\widehat Q_b$ policies of filtering (a) and generative restoration (b) in the dimension sweep at $b = 2$, against the effective number of training instances $n(1-p)^d$ and $n(1-p)^b$, respectively, for $d \in \{6, 8, 10\}$ and $p \in \{0.3, 0.5, 0.7\}$.}
	\label{fig:exact-collapse}
\end{figure}

\subsection{AFA Methods under \AFAITD{}}
\label{sec:main-results}
\begin{figure}[!tb]
	\centering
	\includegraphics[width=\textwidth]{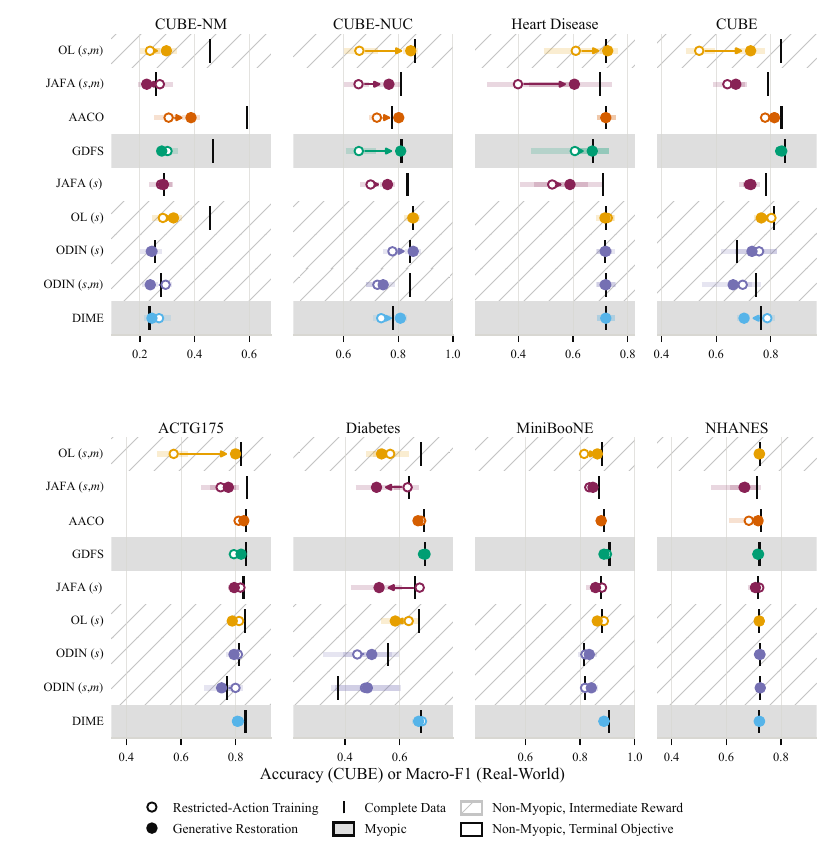}
	\caption{Restricted training vs. generative restoration under MCAR at $p = 0.7$.
		Arrows point from restricted training to generative restoration, and bars are 95\% bootstrap intervals over five splits.
		Methods and datasets are ordered by descending mean missingness damage.
		Gray rows denote myopic methods (\DIME{}, \GDFS{}), hatched rows denote non-myopic methods with intermediate predictive rewards (\OL{}, \ODIN{}), and white rows denote non-myopic acquisition with a terminal prediction objective (\AACO{}, \JAFA{}).}
	\label{fig:main-summary}
\end{figure}

\paragraph{Evaluation Setup and Protocol.}
We use the AFA benchmark framework of \citet{Schutz2025} and adapt it to our \AFAITD{} problem setting. We evaluate six AFA methods: the myopic \DIME{}~\citep{gadgil2024estimating} and \GDFS{}~\citep{covert2023learning} methods, and the non-myopic \AACO{}~\citep{Valancius2024}, \JAFA{}~\citep{Shim2018}, \OL{}~\citep{Kachuee2019}, and \ODIN{}~\citep{ODIN2019} methods. For each method, we train three variants: a complete-data version without any missingness, another with generative restoration applied to the training data beforehand, and a restricted training version with acquisitions restricted to available features.
The non-myopic RL-based methods \JAFA{}, \OL{}, and \ODIN{} are each represented by two separate methods, depending on whether they include the missingness mask in the state (e.g., \OL{} $(s,m)$) or not (e.g., \OL{} $(s)$), corresponding to the filtering and aliasing approaches, respectively.
For restoration, we fit a label-conditioned Partial Variational Autoencoder (PVAE)~\citep{ma2019eddi,ODIN2019} on incomplete training data and replace only missing values with a joint sample conditioned on the observed features and label.

We use the benchmark's synthetic CUBE~\citep{ODIN2019}, CUBE-NM~\citep{Schutz2025}, and CUBE with non-uniform costs (CUBE-NUC) datasets.
For real-world datasets, we use ACTG175~\citep{Hammer1996ACTG175}, Diabetes~\citep{Kachuee2019}, NHANES mortality (NHANES)~\citep{Erion2022CoAI}, MiniBooNE~\citep{Roe2005MiniBooNE}, and Heart Disease~\citep{Janosi1989HeartDisease}.
We synthetically induce MCAR, MAR, logistic MNAR, and self-masking MNAR at rates $p \in \{0.3, 0.5, 0.7\}$.
We use five dataset splits, pairing restricted and restored treatments within each split.
The predictor is trained before masking and held fixed across treatments, isolating the effect on acquisition.
Missingness masks remain fixed across epochs and acquisition episodes; evaluation instances are fully observed.
We report mean accuracy on the synthetic datasets and mean macro-F1 on the real-world datasets, at each dataset's largest predefined acquisition budget under \Cref{eq:afaopt}.
RL training budgets are selected by convergence on complete data and held fixed across treatments.
Exact budgets, optimization settings, missingness mechanisms, and method adaptations are given in \Cref{app:experimental-details}.

\Cref{fig:main-summary} illustrates our primary findings under MCAR at $p = 0.7$. We only report results for the highest missingness rate as we observe the same patterns, but weaker, at the other missingness rates.
The figure orders both methods and datasets according to their mean \emph{missingness damage}: the difference in performance between the complete-data version (vertical line) and the restricted version (hollow circle).
First, we note that the myopic method \DIME{} is the least damaged, and the three most damaged methods are all non-myopic, consistent with \Cref{prop:restriction}.
Second, restoration (filled circle) recovers part of the damage across a variety of methods and datasets, indicating that the advantage that \Cref{thm:data-requirements} predicted in the tabular setting can hold true for methods using function approximation on real-world datasets.
However, restoration is far from a ``free lunch'': in some configurations, restoration performs worse than restricted training.
This is because missingness can remove redundant or noisy features, making the problem easier; restoring all features then returns the method to the harder original problem.
Among the datasets, CUBE-NM is the most affected, aligning with its non-myopic structure~\citep{Schutz2025}.
In contrast, we note that the two least affected datasets are the real-world MiniBooNE and NHANES mortality datasets.
We further investigate the structure of all eight datasets and its relation to missingness damage in \Cref{sec:dataset-structure,sec:adaptivity}, report the computational cost of restoration in \Cref{app:compute}, and repeat the comparison under MAR and MNAR in \Cref{app:other-missingness}.

\subsection{Dataset Structure}
\label{sec:dataset-structure}
\Cref{eq:alias-shortcut} shows that aliasing can converge to the wrong policy, but not that every dataset must be damaged by the restricted action space.
If missingness only blocks irrelevant features, then restricting the action space does not change the optimal policy.
Similarly, if several strong acquisition routes are interchangeable, blocking one route can have little effect because another route remains available.
Missingness damage also depends on the acquisition structure of a dataset, for which we introduce two empirical measures.
To compute both quantities, we first split the training data into a selection and a validation set.
Then, we sample $L$ feasible acquisition routes (action sequences) $R_1, \ldots, R_L$ under budget $b$.
Let $V_{\mathrm{sel}}(R)$ and $V_{\mathrm{val}}(R)$ denote the static scores of route $R$ on the selection and validation splits, respectively, and let
\[
	R_\star \in \argmax_{R \in \{R_1, \ldots, R_L\}} V_{\mathrm{sel}}(R), \ V_{\mathrm{static}} \coloneqq V_{\mathrm{val}}(R_\star).
\]

We define \emph{route sensitivity} as
\[
	\Delta_{\mathrm{route}} \coloneqq V_{\mathrm{static}} - \frac{1}{L} \sum_{k = 1}^L V_{\mathrm{val}}(R_k).
\]
A large $\Delta_{\mathrm{route}}$ indicates that the selected route outperforms a typical feasible route, while a small value indicates little separation in their validation scores.
To measure whether high-performing routes use the same acquisitions, let
\[
	w_k \coloneqq \left[V_{\mathrm{sel}}(R_k) - V_{\mathrm{sel}}(\varnothing)\right]_+
\]
denote the positive predictive gain over the empty route.
Treating each route as a set of legal acquisition actions, we define \emph{performance-weighted route overlap} as
\[
	\omega_{\mathrm{route}} \coloneqq \frac{ \sum_{1 \leq i < j \leq L} w_i w_j \dfrac{|R_i \cap R_j|}{|R_i \cup R_j|}}{\sum_{1 \leq i < j \leq L} w_i w_j}.
\]
A large value indicates that predictive gain is concentrated on routes with overlapping acquisition sets, whereas a small value indicates that the gain is distributed across structurally distinct routes.

\input{experiments/results/route_structure}

\Cref{tab:route-structure} shows $V_{\mathrm{static}}$, $\Delta_{\mathrm{route}}$, and $\omega_{\mathrm{route}}$ across the datasets, with additional results available in \Cref{app:other-missingness}.

\paragraph{Observations on \Cref{tab:route-structure}.}
CUBE-NM has by far the lowest $V_{\mathrm{static}}$, indicating that a static route does not yield good performance.
It also has the lowest $\omega_{\mathrm{route}}$, so what predictive gain there is comes from structurally distinct routes.
In contrast, NHANES has the largest $\omega_{\mathrm{route}}$ and the smallest $\Delta_{\mathrm{route}}$, indicating that we should expect low missingness damage.
We want to highlight MiniBooNE's interesting structure; it has the highest $V_{\mathrm{static}}$ while also having the second-smallest $\omega_{\mathrm{route}}$ and a small $\Delta_{\mathrm{route}}$, indicating that there exist a few top-performing routes with minimal overlap.

\subsection{Missingness Damage and Adaptive Structure}
\label{sec:adaptivity}
By \Cref{prop:restriction}, MCAR leaves one-step values unchanged and can only lower multi-step values.
Thus, only methods that exploit multi-step acquisitions can have their exact values damaged.
\Cref{tab:adaptivity} compares each method's complete-data score with the static route $V_{\mathrm{static}}$ of \Cref{tab:route-structure}.
\DIME{} and \ODIN{} score below the static route on average, and on CUBE-NM, where the other methods are damaged the most, they score $0.24$ and $0.27$ against $V_{\mathrm{static}} = 0.32$.
In contrast, \GDFS{}, \AACO{}, and \OL{} exceed the static route and are damaged.
\GDFS{} is myopic, so its exact one-step values are unchanged, but it estimates them from its own acquisition trajectories: a training instance follows the evaluation trajectory to depth $t$ only if all $t$ acquired features are available, which under MCAR occurs with probability $(1-p)^t$, the same factor as in \Cref{thm:data-requirements}.
Restoration removes this factor, in line with its gain for \GDFS{}.
\JAFA{} is the exception, scoring below the static route on average while being damaged on Heart Disease and CUBE-NUC.
Finally, the correlation between margin and damage across method--dataset cells is moderate and driven mainly by CUBE-NM.

\input{experiments/results/adaptivity}

%% file: experiments/results/route_structure.tex
\begin{table}[H]
	\centering
	\caption{Fixed-route structure. Each column ranks the datasets by that quantity alone, largest first, so the three rankings are independent.
		Mean $\pm$ standard error over five dataset splits.
		Static references and route weights are selected on training data; $V_{\mathrm{static}}$ and $\Delta_{\mathrm{route}}$ are reported on validation data.
		Each dataset is at its largest evaluation budget $b$, which is 15 for Diabetes and MiniBooNE, 14 for CUBE-NM, 10 for ACTG175, CUBE and NHANES, 7 for CUBE-NUC and Heart disease.}
	\label{tab:route-structure}
	{\footnotesize
		\setlength{\tabcolsep}{3pt}
		\begin{tabular}{lclclc}
			\toprule
			\multicolumn{2}{c}{$V_{\mathrm{static}}$} & \multicolumn{2}{c}{$\Delta_{\mathrm{route}}$} & \multicolumn{2}{c}{$\omega_{\mathrm{route}}$}                                                         \\
			\midrule
			MiniBooNE                                 & $0.894 \pm 0.004$                             & CUBE                                          & $0.245 \pm 0.010$ & NHANES        & $0.669 \pm 0.000$ \\
			CUBE-NUC                                  & $0.847 \pm 0.019$                             & ACTG175                                       & $0.175 \pm 0.012$ & CUBE          & $0.342 \pm 0.000$ \\
			ACTG175                                   & $0.846 \pm 0.012$                             & CUBE-NUC                                      & $0.174 \pm 0.017$ & ACTG175       & $0.287 \pm 0.001$ \\
			CUBE                                      & $0.824 \pm 0.010$                             & Diabetes                                      & $0.166 \pm 0.001$ & Heart disease & $0.264 \pm 0.003$ \\
			NHANES                                    & $0.731 \pm 0.005$                             & Heart disease                                 & $0.127 \pm 0.015$ & CUBE-NUC      & $0.221 \pm 0.000$ \\
			Heart disease                             & $0.720 \pm 0.019$                             & CUBE-NM                                       & $0.059 \pm 0.010$ & Diabetes      & $0.205 \pm 0.000$ \\
			Diabetes                                  & $0.679 \pm 0.003$                             & MiniBooNE                                     & $0.058 \pm 0.003$ & MiniBooNE     & $0.181 \pm 0.000$ \\
			CUBE-NM                                   & $0.317 \pm 0.020$                             & NHANES                                        & $0.009 \pm 0.002$ & CUBE-NM       & $0.164 \pm 0.000$ \\
			\bottomrule
		\end{tabular}
	}
\end{table}

%% file: experiments/results/adaptivity.tex
\begin{table}[H]
\centering
\caption{Complete-data score, its margin over the static route $V_{\mathrm{static}}$ (\Cref{tab:route-structure}), and missingness damage under MCAR at $p=0.7$, averaged over the eight datasets and, for \JAFA{}, \OL{} and \ODIN{}, over both state representations.
Across method--dataset cells, Pearson $r=0.59$ and Spearman $\rho=0.34$ between margin and damage.}
\label{tab:adaptivity}
{\footnotesize
\begin{tabular}{lccc}
\toprule
Method & Complete & Margin over $V_{\mathrm{static}}$ & Damage \\
\midrule
\DIME{} & $0.705$ & $-0.027$ & $+0.002$ \\
\ODIN{} & $0.667$ & $-0.065$ & $+0.009$ \\
\GDFS{} & $0.745$ & $+0.013$ & $+0.059$ \\
\AACO{} & $0.758$ & $+0.026$ & $+0.062$ \\
\JAFA{} & $0.707$ & $-0.026$ & $+0.072$ \\
\OL{} & $0.745$ & $+0.013$ & $+0.092$ \\
\bottomrule
\end{tabular}
}
\end{table}

%% file: sections/conclusion.tex
\section{Conclusion}
\label{sec:conclusion}
We have investigated the \AFAITD{} problem setting and found that:
(i) in line with our theory (\Cref{prop:restriction}), incomplete training data mainly damages methods that exploit multi-step acquisitions;
(ii) in the tabular setting, the data requirement of filtering grows exponentially with the dimension, whereas that of generative restoration grows only with the acquisition budget (\Cref{thm:data-requirements});
(iii) generative restoration can in practice reduce the resulting performance gap for non-myopic methods.
We hope that our work opens up the development of AFA methods that leverage the theoretical framework of \AFAITD{}.
For example, by leveraging AFA learning and generative learning jointly, it might be possible to restore only the informative features and leave the uninformative ones blocked.

%% file: sections/appendix.tex
\newpage

\appendix

\input{sections/appendix/core-proofs}
\input{sections/appendix/results-and-setup}

%% file: sections/appendix/core-proofs.tex
\section{Proofs}
\label{app:proofs}

\subsection{Proof of \texorpdfstring{\Cref{prop:restriction}}{the restriction proposition}}
\label{app:proof-restriction}

\begin{proof}
	The proof proceeds by induction on the number of remaining acquisitions $k$.
	Under MCAR, we know that $p(y \mid s, m) = p(y \mid s)$, so the stopping values must satisfy
	$$
		\begin{aligned}
			V_0^{\mathrm{train}}(s,m) & =\E_{y\mid s,m}[-\ell(f(s),y)]                       \\
			                          & =\E_{y\mid s}[-\ell(f(s),y)]=V_0^{\mathrm{eval}}(s).
		\end{aligned}
	$$
	This establishes the case for $k = 0$.
	For the induction step, we assume that $V_{k - 1}^{\mathrm{train}}(s', m) \leq  V_{k - 1}^{\mathrm{eval}}(s')$ at every successor state $s'$ reachable under mask $m$.
	Consequently, for all training actions $a \in \Atrain_b(S, m)$, we have
	$$
		\begin{aligned}
			Q_k^{\mathrm{train}}(s,m,a) & =\E_{x_a\mid s,m}\left[V_{k-1}^{\mathrm{train}}(s',m)\right]                         \\
			                            & \leq\E_{x_a\mid s}\left[V_{k-1}^{\mathrm{eval}}(s')\right]=Q_k^{\mathrm{eval}}(s,a),
		\end{aligned}
	$$
	where the inequality step uses the induction hypothesis and MCAR.
	For the value function, we have
	$$
		\begin{aligned}
			V_k^{\mathrm{train}}(s,m) & =\max\left\{V_0^{\mathrm{train}}(s,m),\max_{a\in\Atrain_b(S,m)}Q_k^{\mathrm{train}}(s,m,a)\right\}                   \\
			                          & \leq\max\left\{V_0^{\mathrm{eval}}(s),\max_{a\in\Atrain_b(S,m)}Q_k^{\mathrm{eval}}(s,a)\right\}                      \\
			                          & \leq\max\left\{V_0^{\mathrm{eval}}(s),\max_{a\in\Aeval_b(S)}Q_k^{\mathrm{eval}}(s,a)\right\}=V_k^{\mathrm{eval}}(s),
		\end{aligned}
	$$
	where the last inequality follows from the fact that $\Atrain_b(S, m) \subseteq \Aeval_b(S)$.

	For $k = 1$, we further have
	$$
		Q_1^{\mathrm{train}}(s,m,a)=\E_{x_a\mid s,m}[V_0^{\mathrm{train}}(s',m)]=\E_{x_a\mid s}[V_0^{\mathrm{eval}}(s')]=Q_1^{\mathrm{eval}}(s,a),
	$$
	which concludes the proof.
\end{proof}

\subsection{Proof of \texorpdfstring{\Cref{thm:data-requirements}}{the data requirements theorem}}
\label{app:proof-data-requirements}

\begin{proof}
	Both the filtering and restoration approaches need to estimate two things: the expected stopping values $V_0(s)$ and the acquisition dynamics $p(x_a \mid s)$ for the Bellman recursion, but they differ in how many usable training instances contribute to these estimates.
	The proof will proceed in the following steps:
	\begin{enumerate}
		\item We count the number of usable training instances for both approaches (\Cref{sec:thm_proof_step1}).
		\item We count the total number of estimates each approach needs (\Cref{sec:thm_proof_step2}).
		\item We bound the error of each estimate type (\Cref{sec:thm_proof_step3}).
		\item Finally, we bound the total evaluation regret in terms of errors of the individual estimates (\Cref{sec:thm_proof_step4}).
	\end{enumerate}
	For the bounding steps, we make use of the following two inequalities:

	\paragraph{Multiplicative Chernoff bound}

	Let $X = \sum_{i = 1}^n X_i$ be a sum of independent binary random variables with mean $\mu = \E[X]$.
	For a fixed $\gamma \in (0, 1)$ the multiplicative Chernoff bound is
	\begin{equation*}
		P\left(X < (1 - \gamma) \mu \right) \leq \exp \left(-\frac{\gamma^2 \mu}{2}\right).
	\end{equation*}

	\paragraph{Hoeffding bound}

	Let $\bar Z = \frac 1N \sum_{i = 1}^N Z_i$ be an average of $N$ independent random variables, each taking values in an interval of length $R$.
	For any $t > 0$ the one-sided Hoeffding bound is
	\begin{align*}
		P\left(\bar Z - \E[\bar Z] \geq t \right) & \leq \exp \left(-\frac{2Nt^2}{R^2}\right)
	\end{align*}
	with the two-sided form
	\begin{equation*}
		P\left(|\bar Z - \E[\bar Z]| \geq t \right) \leq 2 \exp \left(-\frac{2Nt^2}{R^2}\right).
	\end{equation*}

	\subsubsection{Count the Number of Usable Training Instances}
	\label{sec:thm_proof_step1}

	Let $N_{s}^\text{gen}$ and $N_{s,a}^\text{gen}$ denote the number of training instances that the restoration approach can use to estimate $V_0(s)$ and $p(x_a \mid s)$, respectively.
	$N_{s}^\text{gen}$ counts the number of instances that match $x_S$ and have every feature in $S$ available, while $N_{s,a}^\text{gen}$ counts instances that additionally have feature $a$ available,
	$$
		\begin{aligned}
			N_s^\text{gen}      & \coloneqq \sum_{i = 1}^n \mathbf1 \{x_S^{(i)} = x_S,\ S\subseteq\bar M^{(i)}\},          \\
			N_{s, a}^\text{gen} & \coloneqq \sum_{i = 1}^n \mathbf1 \{x_S^{(i)} = x_S,\ S\cup\{a\}\subseteq\bar M^{(i)}\}.
		\end{aligned}
	$$

	Under MCAR, the feature values and labels are independent of availability, thus
	$$
		\begin{aligned}
			\E[N_s^\text{gen}]      & = n P_s(1-p)^{|S|},     \\
			\E[N_{s, a}^\text{gen}] & = n P_s(1-p)^{|S| + 1},
		\end{aligned}
	$$
	where $P_s$ is the probability that a random data instance agrees with the observed values $x_S$.
	Every acquisition costs at least one, which means that $|S| \leq b$ and an available next acquisition satisfies $|S| + 1 \leq b$.
	In contrast, the filtering approach can only use complete instances
	\begin{equation*}
		\E[N_s^\text{filt}] = \E[N_{s, a}^\text{filt}] = n P_s(1-p)^d.
	\end{equation*}

	For brevity, we write the common bound for both methods
	$$
		\E[N_s] \geq n q P_s, \ \E[N_{s,a}] \geq nqP_s,
	$$
	where
	$$
		q =
		\begin{cases}
			(1 - p)^d, & \text{for filtering},   \\
			(1 - p)^b, & \text{for restoration}.
		\end{cases}
	$$
	If $N_s = 0$, we set $\widehat V_0(s)$ to a fixed value in $[-\ell_{\max}, 0]$, and if $N_{s,a} = 0$, we set $\widehat p(\cdot \mid s)$ to a fixed distribution on $\mathcal X_a$, e.g., the uniform distribution.

	\subsubsection{Count the Number of Estimates}
	\label{sec:thm_proof_step2}

	With $j$ acquired features, there are $\binom dj$ feature sets and at most $K^j$ possible value assignments.
	Each state $s$ has one stopping-value estimate $\widehat{V}_0(s)$ and at most $d$ acquisition dynamics to estimate $\widehat{p}(\cdot \mid s)$.
	The total number of estimates is therefore at most
	$$
		(d + 1) \sum_{j = 0}^b \binom dj K^j \leq (d + 1) \sum_{j = 0}^b (Kd)^j \leq 2(d + 1) (Kd)^b \leq (Kd)^{b + 2},
	$$
	and we write $J \coloneqq (Kd)^{b + 2}$ for an upper bound on the number of estimates needed.

	\subsubsection{Bound the Error of Each Estimate}
	\label{sec:thm_proof_step3}

	\paragraph{Failure possibilities and confidence allocation.}
	For each stopping-value estimate $\widehat V_0(s)$, failure can occur if
	(i) the number of usable instances is too small, or
	(ii) the estimate deviates too far in either direction.
	For each transition-distribution estimate $\widehat p(\cdot\mid s)$, failure can occur if
	(i) the number of usable instances is too small, or
	(ii) for some $A\subseteq\mathcal X_a$, the empirical probability $\widehat p(A\mid s)$ deviates too far from $p(A\mid s)$.

	There are at most $J$ estimates of both types combined.
	Hence, there are at most $J$ count-failure events, $2J$ stopping-value deviation events, and $2^K J$ transition-subset deviation events.
	Since $K\geq2$, their total number is at most
	\[
		J(1+2+2^K)\leq 2^{K+1}J.
	\]
	Thus, we allocate a failure probability of $\delta/(2^{K+1}J)$ to each event and define
	\[
		\Lambda_J
		\coloneqq
		\log\frac{2^{K+1}J}{\delta}
		=
		(b+2)\log(Kd)+(K+1)\log 2+\log\frac1\delta.
	\]
	Thus, whenever an individual failure probability is bounded by $\exp(-A)$, requiring $A\geq\Lambda_J$ makes it at most $\delta/(2^{K+1}J)$.
	Finally, after deriving the individual bounds, a union bound controls all estimates simultaneously with probability at least $1-\delta$.

	\paragraph{Concentration of the sample counts.}
	The count $N_s$ is a sum of independent binary variables, so we apply the Chernoff bound at $\gamma = 1/2$ with $\mu = \E[N_s] \geq nqP_s$,
	$$
		\begin{aligned}
			P \left(N_s < \frac{n q P_s}{2}\right) & \leq P \left(N_s < \frac{\E[N_s]}{2}\right) \\
			                                       & \leq \exp\left(-\frac{\E[N_s]}{8}\right)    \\
			                                       & \leq \exp\left(-\frac{n q P_s}{8}\right).
		\end{aligned}
	$$
	Hence, whenever $nqP_s \geq 8\Lambda_J$,
	$$
		P\left(N_s < \frac{nqP_s}{2}\right) \leq e^{-\Lambda_J}.
	$$
	Since $\E[N_{s, a}] \geq nqP_s$ as well, the same bound holds for $N_{s, a}$.

	\paragraph{Concentration of the estimates.}
	Now suppose that $nqP_s \geq 8\Lambda_J$.
	On the event $N_s \geq nqP_s/2$, the two-sided Hoeffding bound gives
	$$
		P\left(|\widehat{V}_0(s)-V_0(s)|>t,\,
		N_s\geq\frac{nqP_s}{2}\right)
		\leq2\exp\left(-\frac{nqP_st^2}{\ell_{\max}^2}\right).
	$$
	Choosing
	$$
		t = \ell_{\max}\sqrt{\frac{\Lambda_J}{nqP_s}},
	$$
	makes this probability at most $2e^{-\Lambda_J}$.
	For a transition estimate, we write
	\[
		\Delta(x_a)
		\coloneqq
		\widehat p(x_a\mid s)-p(x_a\mid s)
	\]
	and define its total-variation error by
	\[
		\operatorname{TV}(\widehat p,p)
		\coloneqq
		\frac12\sum_{x_a\in\mathcal X_a}|\Delta(x_a)|.
	\]
	Since $\widehat p(\cdot\mid s)$ and $p(\cdot\mid s)$ both sum to one, $\sum_{x_a}\Delta(x_a)=0$.
	Consequently, the total positive and negative deviations are equal, and
	\[
		\operatorname{TV}(\widehat p,p)
		=
		\sum_{x_a:\,\Delta(x_a)>0}\Delta(x_a)
		=
		\max_{A\subseteq\mathcal X_a}
		\bigl(\widehat p(A\mid s)-p(A\mid s)\bigr).
	\]
	Thus, it is sufficient to control the empirical probability of every subset $A\subseteq\mathcal X_a$.
	Conditional on $N_{s,a}=N$, for any fixed $A$ we have
	\[
		\widehat p(A\mid s)
		=
		\frac1N\sum_{i=1}^N
		\mathbf 1\{x_a^{(i)}\in A\},
	\]
	which is an average of independent binary variables with expectation $p(A\mid s)$.
	The one-sided Hoeffding bound and a union bound over the at most $2^K$ subsets of $\mathcal X_a$ therefore give
	\[
		P\left(
		\operatorname{TV}(\widehat p,p)>t
		\,\middle|\,
		N_{s,a}=N
		\right)
		\leq
		2^K\exp(-2Nt^2).
	\]
	Hence, on the event $N_{s,a}\geq nqP_s/2$,
	\[
		P\left(
		\operatorname{TV}(\widehat p,p)>t,\,
		N_{s,a}\geq\frac{nqP_s}{2}
		\right)
		\leq
		2^K\exp(-nqP_st^2).
	\]
	Taking $t=\sqrt{\Lambda_J/(nqP_s)}$ makes this probability at most $2^Ke^{-\Lambda_J}$.

	\paragraph{A uniform bound over all states.}
	To express both kinds of deviation error (in $\widehat{V}$ and $\widehat{p}$) on the scale of the value function, let
	$$
		e_s\coloneqq\max\left\{|\widehat V_0(s)-V_0(s)|,\ell_{\max}\max_{a\in\Aeval_b(S)}\operatorname{TV}\!\left(\widehat p(\cdot\mid s),p(\cdot\mid s)\right)\right\}.
	$$
	If no acquisition is legal, only the stopping-value error is present.
	The concentration bound above applies when $nqP_s \geq 8\Lambda_J$.
	When instead $nqP_s < 8\Lambda_J$, the sample counts may be small or zero, so we make no concentration claim.
	However, boundedness still gives $e_s \leq \ell_{\max}$.
	The two cases can therefore be combined as
	$$
		e_s\leq\ell_{\max}\min\left\{1,\sqrt{\frac{8\Lambda_J}{nqP_s}}\right\}.
	$$
	Indeed, when $nqP_s < 8\Lambda_J$ the minimum equals $1$, while when $nqP_s \geq 8\Lambda_J$ this is only a constant-factor weakening of the concentration bound above.
	There are at most $J$ estimates of both types combined.
	A union bound over their count failures, stopping-value deviations, and transition-distribution deviations is at most
	$$
		J(1 + 2 + 2^K) e^{-\Lambda_J} \leq 2^{K + 1}J e^{-\Lambda_J} = \delta.
	$$
	Consequently, with probability at least $1-\delta$, the bound on $e_s$ holds at every state with $P_s>0$ simultaneously.

	\paragraph{The error budget in the notation of the theorem.}
	It remains to replace $\Lambda_J$ by the quantity $\Lambda \coloneqq b\log(Kd)+K+\log\frac2\delta$ used in the theorem.
	Since $b\geq1$, $K\geq2$, and $\delta<1$, we have
	$$
		(b+2)\log(Kd)\leq3b\log(Kd),\
		(K+1) \log 2 \leq 2K,\
		\log\frac1\delta\leq\log\frac2\delta.
	$$
	Thus $\Lambda_J\leq3\Lambda$ and $8\Lambda_J\leq24\Lambda$.
	Since $\min\{1, \cdot\}$ is nondecreasing, we obtain
	\begin{equation}
		e_s\leq\ell_{\max}\min\left\{1,\sqrt{\frac{24\Lambda}{nqP_s}}\right\}.
		\label{eq:proof-local-error}
	\end{equation}

	\subsubsection{Bound the Evaluation Regret in Terms of Errors of the Individual Estimates}
	\label{sec:thm_proof_step4}

	\paragraph{How local errors affect value estimates of the initial state.}
	Fix a deterministic policy $\pi$.
	Let $V_k^\pi(s)$ be its true value with at most $k$ acquisitions remaining, and let $\widehat V_k^\pi(s)$ be its value using the empirical estimates.
	Both values lie in $[-\ell_{\max},0]$, since the only reward is the negative loss when the policy stops.
	This also holds at states without training samples because their fallback estimates (\Cref{sec:thm_proof_step1}) lie in the same ranges.
	If the policy stops at $s$, its value error is $|\widehat V_0(s)-V_0(s)|\leq e_s$.
	Otherwise, let $a$ be its chosen acquisition.
	The current value error can be decomposed in terms of transition probability errors and next-state value errors:
	$$
		\begin{aligned}
			\widehat V_k^\pi(s)-V_k^\pi(s) & =\sum_{x_a\in\mathcal X_a}\widehat p(x_a\mid s)\widehat V_{k-1}^\pi(s')-\sum_{x_a\in\mathcal X_a}p(x_a\mid s)V_{k-1}^\pi(s') \\
			                               & =\sum_{x_a\in\mathcal X_a}\bigl(\widehat p(x_a\mid s)-p(x_a\mid s)\bigr)\widehat V_{k-1}^\pi(s')                             \\
			                               & \quad+\sum_{x_a\in\mathcal X_a}p(x_a\mid s)\bigl(\widehat V_{k-1}^\pi(s')-V_{k-1}^\pi(s')\bigr).
		\end{aligned}
	$$
	Since the successor values lie in an interval of length at most $\ell_{\max}$, the definition of total variation gives
	$$
		\left|\sum_{x_a\in\mathcal X_a}\bigl(\widehat p(x_a\mid s)-p(x_a\mid s)\bigr)\widehat V_{k-1}^\pi(s')\right|
		\leq\ell_{\max}\operatorname{TV}(\widehat p,p)\leq e_s.
	$$
	Indeed, subtracting the midpoint of the successor-value range from every value leaves the sum unchanged, after which the bound follows from $\operatorname{TV}(\widehat p,p)=\frac12\|\widehat p-p\|_1$.
	The second sum carries the remaining value errors forward, weighted by the true probabilities.
	Taking absolute values consequently gives
	\begin{equation}
		|\widehat V_k^\pi(s)-V_k^\pi(s)|\leq e_s+\sum_{x_a\in\mathcal X_a}p(x_a\mid s)|\widehat V_{k-1}^\pi(s')-V_{k-1}^\pi(s')|.
		\label{eq:proof-error-recursion}
	\end{equation}

	Starting at the initial state $s_0$ and unrolling the recursion in \Cref{eq:proof-error-recursion},
	\begin{equation}
		|\widehat{V}^\pi_b(s_0)-V^\pi_b(s_0)|\leq\sum_{t=0}^b\sum_{\substack{s\text{ reached after}\\t\text{ acquisitions}}}r_{s} e_s,
		\label{eq:proof-weighted-error}
	\end{equation}
	where $r_{s}$ is the probability that $\pi$ reaches $s$.

	\paragraph{Summing the errors over acquisitions.}
	After $t$ acquisitions, a deterministic policy has observed $t$ feature values, each with at most $K$ possible outcomes.
	Its acquisition tree therefore contains at most $K^t$ states at that depth.
	Reaching a state $s=(x_S,S)$ requires an instance to match $x_S$, so $r_s\leq P_s$.
	Furthermore, for any instance the policy visits at most one state at depth $t$, and hence $\sum_{s: \ |S| = t} r_s \leq 1$.
	The sum may be smaller than one because some instances lead to earlier stopping.
	For the states at depth $t$, these observations give
	$$
		\begin{aligned}
			\sum_{s: \ |S| = t} \frac{r_s}{\sqrt{P_s}}
			 & \leq \sum_{s: \ |S| = t} \sqrt{r_s}                  \\
			 & \leq \sqrt{K^t \sum_{s: \ |S| = t}r_s} \leq K^{t/2},
		\end{aligned}
	$$
	where the last inequality uses Cauchy--Schwarz over at most $K^t$ states.
	Multiplying \Cref{eq:proof-local-error} by $r_s$ and summing gives
	$$
		\begin{aligned}
			\sum_{s: |S| = t} r_s e_s
			\leq
			\ell_{\max}\sqrt{\frac{24\Lambda}{nq}}
			\sum_{s: |S| = t}\frac{r_s}{\sqrt{P_s}}
			 & \leq\ell_{\max}\sqrt{\frac{24K^t\Lambda}{nq}}.
		\end{aligned}
	$$
	Substituting into \Cref{eq:proof-weighted-error} and summing over depths yields
	\begin{equation}
		\begin{aligned}
			|\widehat V_b^\pi-V_b^\pi| & \leq\ell_{\max}\sqrt{\frac{24\Lambda}{nq}}\sum_{t=0}^bK^{t/2} \\
			                           & \leq4\sqrt{24}\,\ell_{\max}\sqrt{\frac{K^b\Lambda}{nq}},
		\end{aligned}
		\label{eq:proof-policy-error}
	\end{equation}
	where
	$$
		\sum_{t=0}^bK^{t/2}=\frac{K^{(b+1)/2}-1}{\sqrt K-1}\leq4K^{b/2},
	$$
	for $K\geq2$.

	\paragraph{From value error to regret.}
	To bound the value estimation error in \Cref{eq:proof-policy-error} by $\varepsilon/2$, it is sufficient to require
	$$
		nq\geq C\frac{\ell_{\max}^2 K^b\Lambda}{\varepsilon^2},
	$$
	where $C>0$ is a numerical constant chosen large enough to absorb the fixed constants.
	Under this condition,
	$$
		|\widehat V_b^\pi-V_b^\pi|
		\leq\frac{4\sqrt{24}}{\sqrt C}\varepsilon.
	$$
	Choosing, for example, $C\geq1536$ ensures that
	$$
		|\widehat V_b^\pi-V_b^\pi|\leq\frac{\varepsilon}{2}.
	$$

	Both finite Bellman problems admit deterministic optimal policies.
	Let $\pi^*$ be optimal for evaluation and let $\pi_j$ be the policy returned by the chosen empirical method.
	Empirical optimality gives $\widehat V_b^{\pi_j}\geq\widehat V_b^{\pi^*}$.
	The regret can be decomposed into a sum of value estimation errors:
	$$
		\begin{aligned}
			\operatorname{Reg}(\pi_j)
			 & =V_b^{\pi^*}-V_b^{\pi_j}                        \\
			 & =\bigl(V_b^{\pi^*}-\widehat V_b^{\pi^*}\bigr)
			+\bigl(\widehat V_b^{\pi^*}-\widehat V_b^{\pi_j}\bigr)
			+\bigl(\widehat V_b^{\pi_j}-V_b^{\pi_j}\bigr)      \\
			 & \leq
			|V_b^{\pi^*}-\widehat V_b^{\pi^*}|
			+|\widehat V_b^{\pi_j}-V_b^{\pi_j}|                \\
			 & \leq\frac{\varepsilon}{2}+\frac{\varepsilon}{2}
			=\varepsilon.
		\end{aligned}
	$$
	The middle term is nonpositive by empirical optimality, and the remaining two terms are each at most $\varepsilon/2$ because the value-error bound holds for every deterministic policy.
	Finally, substituting $q=(1-p)^d$ for filtering and $q=(1-p)^b$ for restoration gives the two stated sample-size bounds.
\end{proof}

\subsection{Proof of \texorpdfstring{\Cref{prop:filter-lower}}{the filtering lower bound}}
\label{app:proof-filter-lower}

\begin{proof}
	A training instance is complete with probability $(1-p)^d$, so by a union bound, no training instance is complete with probability at least $1 - n(1-p)^d$.
	On this event, every estimate $\widehat V_0(s)$ and $\widehat p(\cdot \mid s)$ equals its fallback, so the empirical Bellman recursion returns a fixed policy $\pi_0$ that depends only on the costs, the fallbacks, and the tie-breaking rule.
	Relabeling the unit-cost features leaves all three unchanged, and hence also $\pi_0$.
	Since the optimal policy attains accuracy $1$, the regret of a policy is one minus its accuracy, and we consider the first action of $\pi_0$.
	\begin{itemize}
		\item If $\pi_0$ stops immediately, nothing is observed, and since $y$ is uniform, its accuracy is $1/2$.
		\item If $\pi_0$ acquires the shortcut $x_{2b}$, it spends the whole budget $b$ and its accuracy is $3/4$.
		\item Otherwise, $\pi_0$ acquires a unit-cost feature $x_j$, and since $d \geq 2b + 1$, we relabel the unit-cost features so that $x_j$ is noise.
		      The remaining budget $b - 1$ excludes the shortcut.
		      If $\pi_0$ acquires the context $x_1$, at most $b - 2$ features of the relevant block $B_{x_1 + 1}$ remain affordable.
		      Otherwise, $\pi_0$ observes neither $x_1$ nor the shortcut, so its acquisitions are independent of $x_1$, and it can afford at most one full block, which is the relevant one with probability at most $1/2$.
		      If a bit of the relevant block is unobserved, it is uniform and independent of the observations, so $y$ is uniform given the observations.
		      Thus, the accuracy is at most $\frac{1}{2} \cdot 1 + \frac{1}{2} \cdot \frac{1}{2} = \frac{3}{4}$.
	\end{itemize}
	Hence, the regret of $\pi_0$ is at least $1/4$ in every case.
	Finally, if $n(1-p)^d < 1 - \delta$, then $\operatorname{Reg}(\pi^{\mathrm{filter}}) \geq 1/4$ with probability greater than $\delta$.
\end{proof}

In contrast, \Cref{thm:data-requirements} guarantees the same regret for generative restoration with $n$ scaling as $(1-p)^{-b}$.

%% file: sections/appendix/results-and-setup.tex
\section{Experimental Details \& Additional Results}
\label{app:additional}

\subsection{Experimental Details}
\label{app:experimental-details}

\paragraph{Benchmark and evaluation protocol.}
We use the AFABench benchmarking framework~\citep{Schutz2025} and inherit its dataset implementations, preprocessing, predictors, method architectures, and hard-budget evaluation protocol unless stated otherwise.
We use five dataset splits for every reported setting.
CUBE-NUC contains two independently noisy copies of the CUBE features, one twice as costly as the other.
For each split, the predictor is trained before synthetic missingness is imposed and held fixed for all approaches.

\paragraph{Training missingness.}
For each dataset split, a fixed missingness mechanism is instantiated on the training split, and the missingness masks are fixed across epochs and AFA episodes.
We implement synthetic MCAR, MAR, logistic MNAR, and self-masking MNAR at missingness rates $p \in \{0.3, 0.5, 0.7\}$.
Under MCAR, features are independently unavailable with probability $p$.
For MAR and logistic MNAR, we randomly select 30\% of the features to determine the missingness probabilities of the remaining features through a logistic function.
The intercepts are calibrated to match the requested missingness rate.
The selected features remain available under MAR but are independently masked at the same rate under logistic MNAR.
Under self-masking MNAR, the missingness probability of a feature depends on its own value.
Missingness masks that would remove all acquisitions from an instance are resampled.

\paragraph{Generative restoration.}
We implement generative restoration using the PVAE architecture from~\citet{ODIN2019}, which is based on the Partial VAE of~\citet{ma2019eddi}.
For each training--validation pair with artificially induced missingness, we fit a separate PVAE on the training split and use the validation split for model selection.
During training, random input masks are applied on top of each instance's missingness, and reconstruction loss is computed only for its available features.
We then condition the PVAE on each instance's available features and label, draw a single joint sample of all $d$ features from its decoder, and replace only the unavailable values.
We use latent dimension \(20\), batch size \(128\), Adam with learning rate \(10^{-3}\), KL weight \(0.1\), and sample the additional masking probability uniformly from \([0,0.9]\).
All remaining architecture and optimization settings follow \texttt{AFABench}'s \ODIN{} configuration.

\paragraph{Method adaptations.}
For every method trained on restricted data, acquisitions unavailable in a training instance are removed from its legal action set throughout the trajectory.
We evaluate \ODIN{} in its model-free configuration by disabling the additional synthetic-data generation used by its model-based variant.
For \JAFA{}, \OL{}, and model-free \ODIN{}, we additionally compare two state representations.
In the \emph{aliasing} variant, the availability mask is used only to block illegal actions and is not provided to the learned decision function.
In the \emph{filtering} variant, the availability of the remaining acquisitions is included in the state: \JAFA{} and \OL{} learn \(Q(s,m,a)\), while both the actor and critic of \ODIN{} condition on \(m\).
The variants are otherwise identical.

For \DIME{} and \GDFS{}, we retain the standard AFABench architectures and objectives, but initialize unavailable acquisitions as already exhausted in their selection masks.
\AACO{} requires further adaptation because its nearest-neighbor search compares a query against incomplete training instances.
For a query with acquired set \(S\) and a training instance with availability mask \(\bar m^{(i)}\), we use the shared-support distance
\[
	d_i^2(s)
	=
	\frac{
	\sum_{j\in S}\bar m_j^{(i)}
	\bigl(x_j-x_j^{(i)}\bigr)^2
	}{
	\sum_{j\in S}\bar m_j^{(i)}
	}.
\]
Training instances with no shared support are excluded.
Normalizing by the shared-support size prevents instances with fewer comparable features from appearing artificially close, and candidate losses are evaluated using only features available in each neighbor.

\paragraph{Training, budgets, and metrics.}
All remaining settings follow AFABench, except for RL training lengths selected by convergence on complete data.
This yielded $2{,}000$ batches for \OL{} and $4{,}000$ for \JAFA{} and model-free \ODIN{}, with $256$ frames per batch.
These lengths are fixed across missingness mechanisms, rates, state representations, and restoration treatments.
We report each dataset at its largest predefined hard budget: $10$ for CUBE, ACTG175, and NHANES mortality; $14$ for CUBE-NM; $7$ for CUBE non-uniform and Heart Disease; and $15$ for Diabetes and MiniBooNE.
We report accuracy on the balanced synthetic datasets and macro-F1 on the real-world datasets, with means over the five dataset splits.

\subsection{Pseudocount Smoothing in the Synthetic Controlled Study}
\label{app:pseudocount}
The synthetic controlled study uses a pseudocount of $1/2$ for each binary outcome.
For $N_1$ occurrences of outcome one among $N>0$ compatible instances, the change from the empirical frequency is
$$
	\left|\frac{N_1+1/2}{N+1}-\frac{N_1}{N}\right|=\frac{|N/2-N_1|}{N(N+1)}\leq\frac{1}{2(N+1)}.
$$
The fixed binary predictor's expected zero-one loss is either the label-one probability or its complement, so the same bound applies to the stopping-value estimate.
On the count event of the proof of \Cref{thm:data-requirements} (\Cref{sec:thm_proof_step3}), $N\geq nqP_s/2$.
At a state with sufficient expected support, we have $nqP_s \geq 8\Lambda_J$ and $\Lambda_J \geq \log 24 > 1$, so $\Lambda_J nqP_s > 1$ and the additional error satisfies
$$
	\frac{1}{2(N+1)}\leq\frac{1}{nqP_s}\leq\sqrt{\frac{\Lambda_J}{nqP_s}}.
$$
Adding this to the Hoeffding deviation $\ell_{\max}\sqrt{\Lambda_J/(nqP_s)}$ of that proof at most doubles it, and $4\Lambda_J\leq12\Lambda\leq24\Lambda$.
Thus, \Cref{eq:proof-local-error} still holds.
For smaller counts, including $N=0$, the smoothed probability remains in $[0,1]$, so the same bounded-error argument applies.
The remainder of that proof is unchanged, giving the same sufficient sample-size bounds for the smoothed filtering and restoration estimates used in the synthetic controlled study.

\subsection{Computational Cost of Restoration}
\label{app:compute}

Generative restoration introduces an additional model, but that model is trained once per dataset split and missingness setting and then reused across downstream AFA methods.
Across the runs in \Cref{fig:compute}, the median generative-to-restricted wall-clock ratio is $1.15$; $59\%$ of pairs are within $1.25\times$ and $74\%$ within $1.5\times$.
The amortized PVAE fitting and restoration steps contribute a median of $29.3$ seconds per trained method.
Thus, the additional computational cost is typically modest relative to downstream AFA training, although the overhead varies substantially across methods.

\begin{figure}[tbp]
	\centering
	\includegraphics[width=\textwidth]{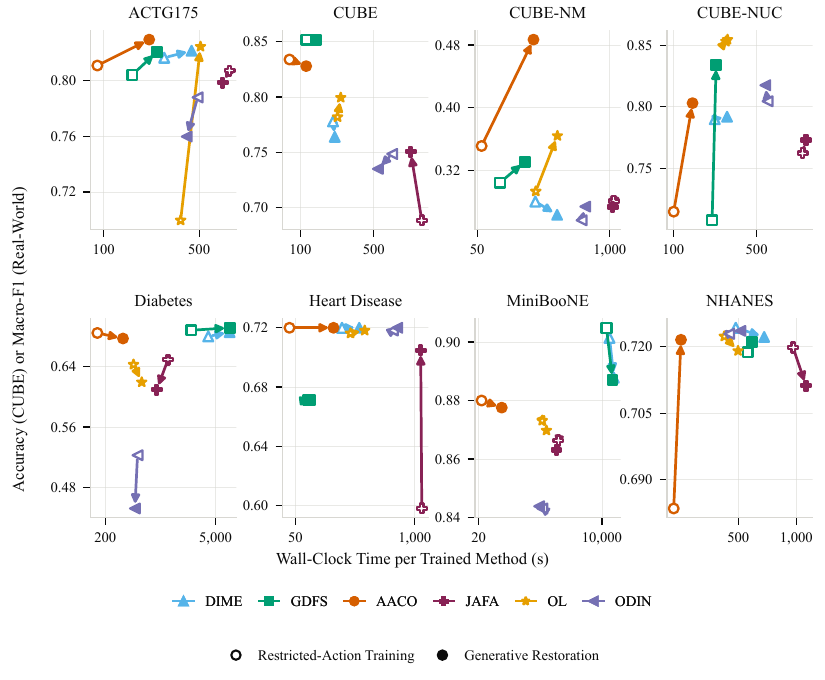}
	\caption{Predictive performance against wall-clock time per trained method.
		Each arrow points from restricted training to generative restoration for the same dataset and method.}
	\label{fig:compute}
\end{figure}

\subsection{Results beyond MCAR}
\label{app:other-missingness}

The main text focuses on MCAR, where the theoretical comparison is sharpest.
\Cref{fig:summary-grid} reports the restricted training and restoration comparison for all four mechanisms (MCAR, MAR, MNAR (logistic), and MNAR (self-masking)).
We observe similar patterns as in \Cref{fig:main-summary}, namely that missingness can have little effect in some settings and substantial effect in others, and restoration does not uniformly improve every method.

\begin{figure}[p]
	\centering
	\includegraphics[width=\textwidth,height=\dimexpr\textheight-3\baselineskip\relax,keepaspectratio]{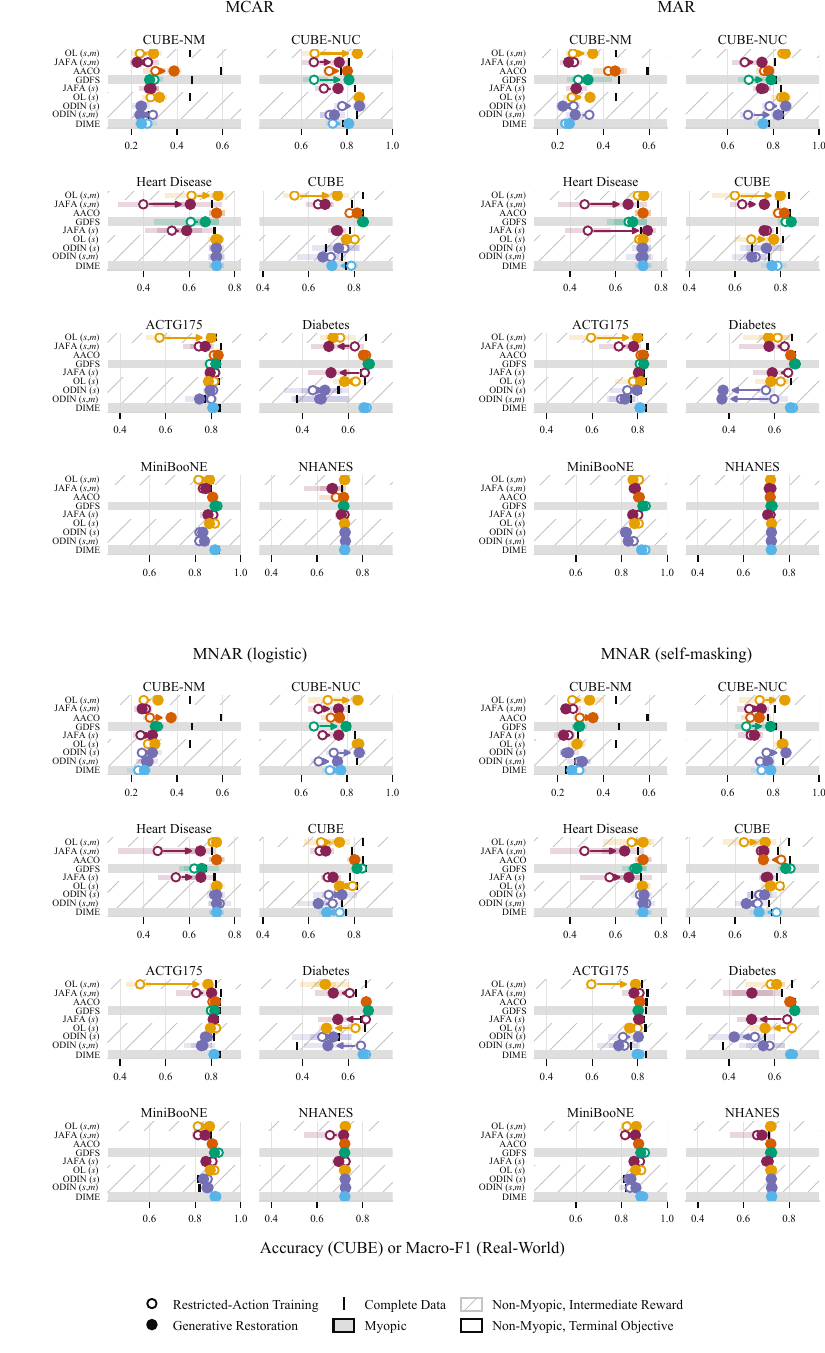}
	\caption{Restricted training against generative restoration at $p=0.7$ under all four missingness mechanisms.}
	\label{fig:summary-grid}
\end{figure}

To separate the amount of damage caused by missingness from the amount recovered by restoration, \Cref{fig:law-grid} plots restoration gain against missingness damage for each method and mechanism.
The fitted rays use cells with damage at least $0.01$, so their slopes summarize recovery only where there is a non-negligible gap to recover.

\begin{figure}[p]
	\centering
	\includegraphics[width=\textwidth]{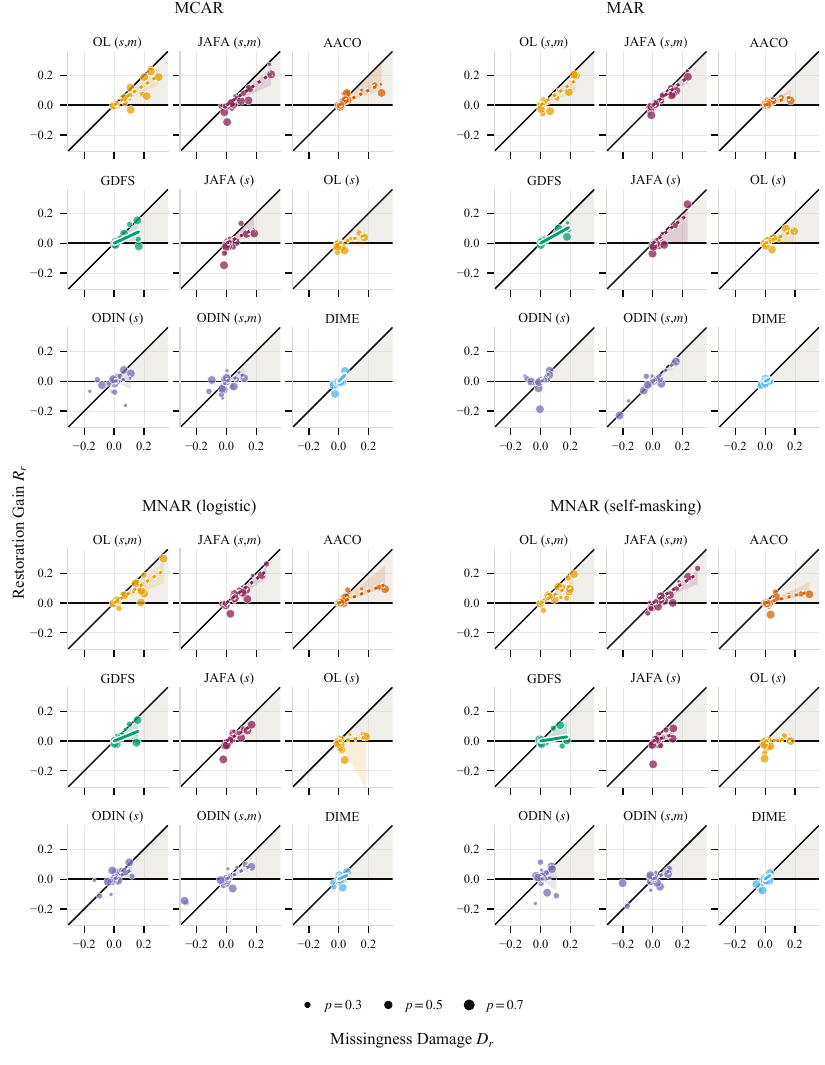}
	\caption{Restoration gain against missingness damage, one panel per method and one block per mechanism, over all eight datasets and three rates.}
	\label{fig:law-grid}
\end{figure}